\documentclass[lettersize,journal]{IEEEtran}

\usepackage{amsmath,amsfonts}
\usepackage{algorithm}
\usepackage{algorithmic}
\usepackage{array}
\usepackage{dsfont}
\usepackage{textcomp}
\usepackage{stfloats}
\usepackage{url}
\usepackage{verbatim}
\usepackage{graphicx}
\usepackage{cite}
\usepackage{xcolor}
\usepackage{gensymb}
\usepackage{subcaption}
\usepackage{float}
\usepackage{enumitem}
\usepackage{booktabs}
\usepackage{multirow}
\usepackage{comment}
\usepackage{framed}
\usepackage{amssymb}
\usepackage{amsthm}
\usepackage{geometry}
\usepackage{amsthm}
\usepackage{mleftright}

\title{Hybrid Feedback–Guided Optimal Learning for Wireless
Interactive Panoramic Scene Delivery}

\newtheorem{remark}{Remark}

\newcommand{\bE}{\mathds{E}}

\newcommand{\ol}{\overline}

\newcommand{\id}{\mathds{1}}
\newcommand{\Beta}{\operatorname{Beta}}

\def\argmax{\operatornamewithlimits{arg\,max}}

\newcommand{\KL}[2]{D_{\mathrm{KL}}\bigl(#1 \,\|\, #2\bigr)}
\newcommand{\prob}{\mathop{\mathrm{Pr}}}  
\newcommand{\probprime}{\mathop{\mathrm{Pr}'}}

\author{Xiaoyi Wu,~\IEEEmembership{Graduate Student Member,~IEEE}, Juaren Steiger,~\IEEEmembership{Member,~IEEE}, Bin Li,~\IEEEmembership{Senior Member,~IEEE} 
and R. Srikant,~\IEEEmembership{Fellow,~IEEE}
}

\begin{document}

\maketitle
\newtheorem{theorem}{Theorem}
\newtheorem{lemma}{Lemma}
\newtheorem{claim}{Claim}
\newtheorem{proposition}{Proposition}
\newtheorem{corollary}{Corollary}
\newtheorem{definition}{Definition}
\newtheorem{assumption}{Assumption}
\newtheorem{remarks}{Remarks}
\bstctlcite{IEEEexample:BSTcontrol}
\begin{abstract}
Immersive applications such as virtual and augmented reality impose stringent requirements on frame rate, latency, and synchronization between physical and virtual environments. To meet these requirements, an edge server must render panoramic content, predict user head motion, and transmit a portion of the scene that is large enough to cover the user viewport while remaining within wireless bandwidth constraints. Each portion produces two feedback signals: prediction feedback, indicating whether the selected portion covers the actual viewport, and transmission feedback, indicating whether the corresponding packets are successfully delivered. Prior work models this problem as a multi-armed bandit with two-level bandit feedback, but fails to exploit the fact that prediction feedback can be retrospectively computed for all candidate portions once the user head pose is observed. As a result, prediction feedback constitutes full-information feedback rather than bandit feedback. Motivated by this observation, we introduce a two-level hybrid feedback model that combines full-information and bandit feedback, and formulate the portion selection problem as an online learning task under this setting. We derive an instance-dependent regret lower bound for the hybrid feedback model and propose AdaPort, a hybrid learning algorithm that leverages both feedback types to improve learning efficiency. We further establish an instance-dependent regret upper bound that matches the lower bound asymptotically, and demonstrate through measurements on an end-to-end testbed that AdaPort outperforms state-of-the-art learning-based baselines as well as the heuristic minimum scene delivery scheme.
\end{abstract}

\section{Introduction}
\label{sec:intro}  
The evolution of immersive technologies, including extended reality (XR), which includes both virtual reality (VR) and augmented reality (AR), holographic displays, and volumetric capture, is redefining digital interaction across multiple domains, such as gaming, telepresence, remote collaboration, and immersive training. These systems provide users with highly responsive and interactive virtual environments that continuously adapt to their movements and perspectives in real time. In practical deployments, panoramic visual content is commonly streamed wirelessly from an edge server to a head-mounted display (HMD), such as a VR headset, since strict weight and power constraints limit the onboard rendering capabilities of HMDs for high-fidelity 3D graphics. Unlike traditional video streaming, immersive panoramic experiences are particularly sensitive to latency and misalignment, which may result in user discomfort or motion sickness. To ensure a seamless experience, tight synchronization between the user’s real-world \textit{head pose} (i.e., position and orientation) and the corresponding virtual viewpoint is essential. Consequently, the edge server cannot rely solely on the latest pose information and must instead \textit{predict} the user’s current head pose. Based on this prediction, the server proactively transmits a region of the panoramic scene that is sufficiently large to fully cover the user’s \textit{viewport}, while compensating for inevitable head-motion prediction errors.  

A straightforward approach would be for the edge server to transmit a portion covering the full 360$\degree$ panoramic scene. However, such a strategy is impractical due to constraints imposed by the wireless channel and stringent real time streaming requirements. Compared to conventional video delivery, panoramic video streaming typically demands four to six times higher bandwidth \cite{bao2016}. As a result, transmitting the entire scene can easily cause severe congestion. Moreover, user interaction in immersive panoramic environments induces frequent and rapid head movements, which in turn lead to more pronounced channel dynamics than those experienced by stationary users. This increased variability further exacerbates packet loss.
Real time latency constraints also play a critical role. For instance, Meta recommends a target of 60 frames-per-second (FPS) for media applications, while interactive applications are required to sustain at least 72 FPS on Meta Quest headsets \cite{metadocs}. In addition, user studies conducted by Wang et al. \cite{wang2023framerate} indicate that frame rates around 120 FPS significantly reduce virtual reality sickness. Under such stringent timing requirements, the edge server often lacks sufficient opportunity to retransmit lost packets within a single frame interval. Consequently, the server must carefully select a transmission region that is large enough to reliably cover the user viewport, yet sufficiently small to be delivered successfully over the wireless channel within one frame duration.

If the dynamics of the user's head motion and the wireless channel were known in advance, the edge server could calculate the optimal delivery portion. However, in practice, these statistics are unknown and must be learned in real-time. A natural approach is to model the problem as a multi-armed bandit, where each arm is a delivery portion and a Bernoulli reward is accrued in a timeslot if the chosen delivery portion results in the user successfully seeing their entire viewport. In general, the rewards are non-i.i.d. due to the nonstationary user head motion and panoramic content. However, stochastic bandit algorithms, such as KL-UCB, are shown to empirically converge in real-world panoramic scene delivery experiments \cite{gupta2022online}. Moving beyond this basic formulation, Chen et al. \cite{chen2020thompson} observe that the overall reward naturally decomposes into two distinct feedback signals. The first corresponds to the prediction outcome, which captures whether the selected delivery portion indeed covers the user actual viewport. The second corresponds to the transmission outcome, which reflects whether all packets associated with the selected portion are successfully delivered over the wireless channel. In their framework, prediction and transmission outcomes are treated as separate bandit feedback observations, leading to what they term \textit{``two-level bandit feedback''}. Throughout this paper, we denote this setting as 2/B/B, while the conventional single-level bandit feedback is denoted as 1/B. \footnote{This $n/x/x$ notation, where $n$ indicates the number of different feedback signals, and the remaining $x$'s indicate the type of each feedback signal, is inspired by Kendall's notation from queueing theory.} In the presence of two underlying bandit feedback signals, the 1/B setting corresponds to the scenario in which only the product of the two feedback signals is observable.

However, existing works \cite{chen2020thompson,gupta2022online} overlook an important structural property of the problem. Once the edge server receives the user head pose, it can retrospectively determine whether the user viewport would have been covered by every possible delivery portion, rather than only the portion selected in the previous time slot. As a result, the prediction outcome is not a bandit signal but instead constitutes full-information feedback. The learning problem therefore exhibits two-level feedback in which one level provides full-information feedback and the other provides bandit feedback. We refer to this feedback structure as 2/F/B.
In this paper, we investigate how exploiting this richer feedback structure can improve learning efficiency for wireless interactive panoramic scene delivery. In particular, we demonstrate the advantages of leveraging 2/F/B feedback over conventional 2/B/B and 1/B feedback models. Our main contributions are summarized as follows:
\begin{enumerate}
    \item We formulate the problem of maximizing cumulative successful viewport delivery of an interactive panoramic scene as an online learning problem with two-level full-information and bandit (2/F/B) feedback. To our knowledge, no prior work studies a reward formed as the product of two component signals observed under different feedback types.
    \item We derive an instance-dependent lower bound on the regret under 2/F/B feedback (Theorem~\ref{theorem:lower_bound}), which is shown to be significantly smaller than the corresponding lower bounds for 2/B/B and 1/B feedback for some problem instances. This suggests that incorporating full-information feedback can significantly improve learning efficiency.
    \item We design the \textit{Hybrid Feedback-Driven Learning for Adaptive Portion Selection (AdaPort)} algorithm: a novel online learning algorithm that leverages 2/F/B feedback by using the empirical mean estimate for the full-information feedback and the Thompson sample for the bandit feedback.
    \item We derive an instance-dependent upper bound (Theorem~\ref{theorem:upper_bound}) for AdaPort that matches the lower bound asymptotically, thus showing that the algorithm is asymptotically optimal. We also illustrate numerically the gap between the 2/F/B lower bound constant and those under 2/B/B and 1/B feedback.
    \item 
    To validate the practical effectiveness of AdaPort, we evaluate it on data measured from an end-to-end panoramic scene delivery testbed, in which both feedback signals are produced by the same live wireless session. We compare AdaPort against state-of-the-art stochastic bandit portion selection algorithms that use 2/B/B and 1/B feedback, and also consider a sliding-window variant of AdaPort, which we call \textit{SW-AdaPort}, to examine the impact of temporal dependence and nonstationarity in the testbed environment. SW-AdaPort restricts its estimates to recent observations and therefore adapts more quickly when the underlying statistics change. Our experiments show that SW-AdaPort can outperform AdaPort when the window size is well chosen, while overly small or large windows can lead to worse performance than AdaPort. On this testbed AdaPort improves on both 2/B/B and 1/B feedback. Its margin over 1/B feedback is smaller, for a reason we identify in Section~\ref{section:empirical_evaluation}.
    
\end{enumerate}

This work extends our conference version \cite{wu2025optimal} in the following
aspects. (1) More detailed proofs for Theorem \ref{theorem:lower_bound} and Theorem \ref{theorem:upper_bound} are included, and deeper intuition regarding the theoretical results is provided. (2) The evaluation is rebuilt on an end-to-end panoramic scene delivery testbed, in which the prediction outcome and the transmission outcome are produced by the same live wireless session rather than composited from separately collected traces. This allows us to validate that our approach holds up in a real system, where the i.i.d. and independence assumptions used for theoretical analysis may be violated. (3) We now consider a sliding-window variant of AdaPort, called SW-AdaPort, in order to better target the nonstationary environment.

\emph{Note on Notation}: We use bold and script font of a variable to denote a vector and a set, respectively. We use $a\wedge b$ to denote the $\min\{a,b\}$. 
We use $[N] \triangleq \{1, 2, \ldots, N\}$ to denote the set of the first $N$ positive integers.
We use $f(x) = o(g(x))$ to denote that $\lim_{x \to \infty} f(x)/g(x) = 0$, and  
$f(x) = O(g(x))$ to denote that $\limsup_{x \to \infty} f(x)/g(x) < \infty$,  
for positive functions $f$ and $g$.
We write $f(x) = \Theta(g(x))$ if there exist constants $c_1, c_2 > 0$ and $x_0$ such that  
$c_1 g(x) \leq f(x) \leq c_2 g(x)$ for all $x \geq x_0$. $ F^{B}_{n,p}(\cdot) $ denotes the cumulative distribution function (CDF) of the binomial distribution with $n$ trials and success probability $p$. We use $\Beta(\alpha, \beta)$ to denote the Beta distribution with parameters $\alpha$ and $\beta$ and $F^{\text{Beta}}_{\alpha,\beta}(\cdot)$ to denote the CDF of this Beta distribution. $d(p_1,p_2) \triangleq p_1\log\mleft(\frac{p_1}{p_2}\mright) + {(1-p_1)\log\mleft(\frac{1-p_1}{1-p_2}\mright)}$ is the KL-divergence between the $\text{Bernoulli}(p_1)$ and $\text{Bernoulli}(p_2)$ distributions, with the conventions $0 \log 0 = 0$, so that $d(p,1) = \infty$ for $p<1$ and $d(p,0)=\infty$ for $p>0$, and $1/\infty = 0$. 
1/B refers to the standard (one-level) bandit feedback, 2/B/B refers to two-level bandit feedback, and 2/F/B refers to two-level feedback where one level is full-information, and the other level is bandit feedback.

\section{related work}
\label{section:related_work}

In this section, we review the related work on interactive panoramic scene delivery, and on online learning under full-information feedback and under bandit feedback. 

\subsection{Interactive Panoramic Scene Delivery}
Panoramic scene delivery places significantly greater demands on network bandwidth than traditional video streaming. To address these challenges and enhance user experience, numerous studies have explored strategies for optimizing panoramic content transmission (e.g., \cite{liu2020firefly, qian2018flare}). For example, Qian et al.~\cite{qian2018flare} introduced an adaptive scheme that transmits only the portion of the panoramic scene aligned with the predicted user viewport, effectively reducing bandwidth consumption.
However, these early approaches primarily relied on heuristic methods without providing theoretical performance guarantees. In response, subsequent research has introduced multi-armed bandit frameworks to the problem of panoramic scene delivery, aiming to provide a more principled and theoretically grounded understanding. Notably, recent studies have begun to explore how to effectively utilize the two distinct feedback signals inherent in this setting: motion prediction and wireless transmission. For example, Chen et al. \cite{chen2020thompson} proposed the Two-Level Thompson Sampling algorithm to optimize viewport selection and maximize system throughput, while Gupta et al. \cite{gupta2022online} developed a Two-Level KL-UCB approach for the same objective.
In parallel, other works have addressed multi-user scenarios (e.g., \cite{chakareski2020viewport, chen2021motion, wu2024}) and multi-objective learning  (e.g. \cite{juaren2025learning}) in panoramic scene delivery. Nevertheless, prior studies have generally overlooked a key observation: the prediction outcome provides a full-information feedback signal, while the wireless transmission outcome yields bandit feedback. Thus, the problem naturally exhibits a two-level hybrid feedback structure, combining both full-information and bandit feedback, which has not been explicitly modeled or exploited in existing work.

\subsection{Feedback Models in Online Learning}
In the online learning setting, bandit feedback refers to the scenario where the player observes the reward only for the arm selected at each timeslot, while the rewards of all other actions remain unobserved. In contrast, full-information feedback provides the player with the rewards of all available arms at each timeslot, regardless of the action selected. This richer feedback enables more informed learning strategies. In the context of bandit feedback, several classical algorithms have been extensively studied, including UCB \cite{auer2002finite}, KL-UCB \cite{garivier2011kl}, and Thompson Sampling \cite{agrawal2017near}. While most foundational work focuses on this bandit feedback setting, there also exist a number of studies exploring full-information feedback in both stochastic and non-stochastic bandit frameworks (e.g., \cite{zhao2019stochastic, alon2017nonstochastic, yan2023online, zhang2023online, liu2018information, cheng2023understandingrolefeedbackonline}).
For instance, Zhao et al. \cite{zhao2019stochastic} examined a one-sided full-information stochastic bandit problem, where selecting an arm yields a reward from an unknown distribution while also revealing feedback from all arms located on one side of the selected arm. Alon et al. \cite{alon2017nonstochastic} introduced a partial-information model for online learning that generalizes and interpolates between the classical full-information and bandit settings.

 A broader and extensively studied direction concerns online learning with structured feedback over the action space, which interpolates between the bandit and full-information extremes. In the \emph{feedback graph} model, playing an arm additionally reveals the rewards of its neighbors in a graph, recovering bandit feedback for an empty graph and full information for a complete graph \cite{alon2017nonstochastic, mannor2011bandits}. The closely related \emph{side-observation} model studies stochastic bandits in which selecting an arm yields side observations of related arms \cite{mannor2011bandits, caron2012leveraging}. More generally, \emph{partial monitoring} decouples the incurred loss from the observed signal through a feedback matrix, subsuming both bandit and full-information feedback as special cases \cite{bartok2014partial}. Our setting is also related to \emph{semi-bandit} feedback in combinatorial bandits, where the learner observes the outcomes of the individual components of the selected action \cite{audibert2014regret}. Our 2/F/B model can be cast as a partial-monitoring problem: the prediction outcome is observed for every candidate portion once the user's actual viewport becomes known, whereas the transmission outcome is observed only for the portion that was actually sent. What distinguishes our setting from the feedback graph and side-observation models is that the fully observed signal is a \emph{factor} of the reward rather than the reward itself, and that the outcome distribution has the factorized form of Section~\ref{section:system_model}.

 To the best of our knowledge, no prior work studies a reward formed as the product of two component signals observed under different feedback types, one full-information and one bandit. Our contribution is the explicit instance-dependent analysis of this factorized model: an asymptotically matching regret constant, together with bounded expected exploration for arms whose prediction rate is too low for them to be optimal. General partial-monitoring results place no structure on the outcome distribution and do not directly provide either.

\section{System Model}
\label{section:system_model}

We consider a system in which a single user interacts with an immersive panoramic scene that is delivered wirelessly from an edge server. The panoramic scene can be conceptualized as video content projected onto the inner surface of a virtual sphere surrounding the user's head. The user's viewport, i.e., the visible portion of the scene, which typically constitutes around 20\% of the sphere's surface, is a rectangle centered at a fixed origin point corresponding to the user's view direction. The content on the sphere's surface changes as the user moves through the scene and rotates their head, and also changes according to the dynamics of the virtual environment. Using a standard map projection technique, such as equirectangular projection, the spherical panoramic content is transformed into a two-dimensional finite grid of tiles \footnote{A tile represents the atomic unit for image encoding/decoding.}. A \textit{delivery portion} is a rectangular formation of tiles centered at and with dimensions larger than the user's viewport. We assume there are $N$ fixed delivery portions that the edge server can choose from in each timeslot. Note that each timeslot  $t=1,2,\ldots$ corresponds to a video frame and the system operates at 60-120 timeslots per second depending on the application's framerate requirement.


\subsection{Online Learning under 2/F/B Feedback}

After the edge server selects the delivery portion $i(t) \in [N]$ in timeslot $t$, the content corresponding to that delivery portion is sent to the user as a sequence of packets. Note that in general, the content may not be predetermined and may be updated in real-time according to the user's inputs (e.g., consider a VR video game). Therefore, we can only use a limited playback buffer because the content in timeslot $t$ may become immediately irrelevant in timeslot $t+1$. Therefore, the content occupying the user's viewport in timeslot $t$ must be sent from the server and received by the user by the end of timeslot $t$. Additionally, due to the short timeslot length, we do not consider any packet retransmissions. However, we assume a packet sent from the user to the edge server, called the \textit{ACK packet}, containing the transmission acknowledgment, head pose, and other inputs from the previous timeslot is small enough that it can be reasonably assumed to be reliably delivered.  
\subsubsection{Prediction Outcome as Full-Information Feedback}
\begin{figure}[htbp]
  \centering
  \includegraphics[width=0.38\textwidth, height=0.13\textwidth]{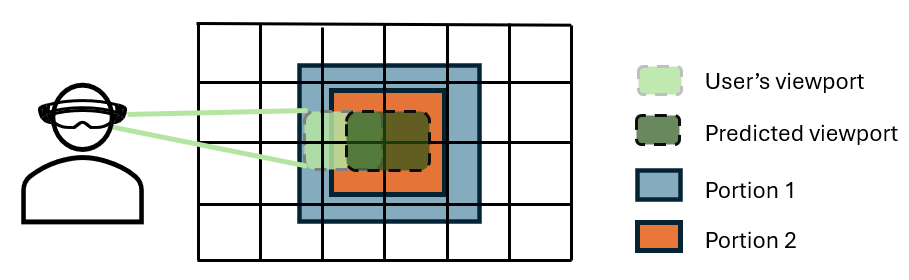}
  \caption{Relationship between the user's viewport, the predicted viewport, and delivery portions.}
  \label{fig:system_model}
\end{figure}

At the beginning of timeslot $t + 1$, the server will have received the user's ACK packet from timeslot $t$. Therefore, the server knows the user's predicted head pose, as well as their actual head pose from timeslot $t$, and can calculate whether or not each portion $i \in [N]$ would have covered their viewport. We use $X_i(t)=1$ to indicate that portion $i$ can cover the user's viewport in timeslot $t$ and $X_i(t)=0$ otherwise. 
As illustrated in Fig.\ref{fig:system_model}, given the user's actual viewport (light green), portion~1 fully covers it, i.e., $X_1(t) = 1$, whereas portion~2 does not, i.e., $X_2(t) = 0$.
Then the server observes the entire vector $\mathbf{X}(t) \triangleq (X_i(t))_{i=1}^N$ by the beginning of timeslot $t+1$, i.e. the server receives full-information feedback for the prediction outcome. We assume $\mathbf{X}(t)$ is i.i.d. over time 
and each component has unknown rate $\alpha_i \triangleq \bE[X_i(t)]$. 
However, note that $\mathbf{X}(t)$ is component-wise dependent, i.e. for a fixed timeslot $t$ and two portions $i\neq j$, $X_i(t)$ and $X_j(t)$ are not independent.
As previously mentioned, they are actually both functions of two random variables, namely, the edge server's prediction of the user's head pose from timeslot $t$ and the user's actual head pose from timeslot $t$.

\subsubsection{Transmission Outcome as Bandit Feedback}

At the beginning of timeslot $t+1$, after receiving the user's ACK packet from timeslot $t$, the server knows whether or not the transmission of delivery portion $i(t)$ was successful. We use $Y_i(t) = 1$ to indicate that the transmission of delivery portion $i$ would be successful if it were selected in timeslot $t$, and $Y_i(t) = 0$ otherwise. Then the server observes the transmission outcome $Y_{i(t)}(t)$ as bandit feedback. In this paper, we consider the transmission to have failed if any packet constituting the delivery portion $i(t)$ is lost. However, the problem can easily be extended to consider a failed transmission to occur only when a minimum portion of the packets are lost. 
As with the prediction outcome, we assume $\mathbf{Y}(t) \triangleq (Y_i(t))_{i=1}^N$ is i.i.d. over time, but unlike the prediction outcome,  we assume it is component-wise independent, with each component having an unknown rate $\beta_i \triangleq \bE[Y_i(t)]$.
Note that, as illustrated in Fig.\ref{fig:system_model}, all candidate portions encompass the predicted viewport (dark green), being centered on it and expanded outward to varying extents.
To summarize the stochastic assumptions: the pairs $(\mathbf{X}(t), \mathbf{Y}(t))$ are i.i.d.\ across timeslots, and in each timeslot $\mathbf{Y}(t)$ is independent of $\mathbf{X}(t)$ and has independent $\text{Bernoulli}(\beta_i)$ components. The distribution of $\mathbf{X}(t)$ is unknown and its components may be dependent. Any randomization internal to the selection algorithm is independent of $(\mathbf{X}(t), \mathbf{Y}(t))_{t \geq 1}$.

We use $Z_i(t)=1$ to indicate that the user would be able to successfully view the desired content if portion $i$ were selected in timeslot $t$, and $Z_i(t)=0$ otherwise. We also refer to $Z_i(t)$ as the \textit{``throughput"} or \textit{``reward"} for portion $i$ in timeslot $t$. Note that $Z_i(t)=1$ only when both the prediction and the transmission are successful, i.e. $Z_i(t)=X_i(t) Y_i(t)$. Our objective is to design an adaptive portion selection algorithm that maximizes the expected cumulative throughput $\sum_{t=1}^T \bE\mleft[ Z_{i(t)}(t) \mright]$ up to a time horizon $T$ when the prediction and transmission rates $(\alpha_i, \beta_i)_{i=1}^N$ are unknown. If the parameters were known, the optimal policy would be to always play the arm  $i^* \in \argmax_{i \in [N]} \alpha_i \beta_i$, which we assume to be unique. Therefore, the regret is given by 


\begin{align}
\label{system_model_regret}
R(T) & \triangleq 
\sum_{t=1}^T \bE\mleft[ Z_{i^*}(t) - Z_{i(t)}(t) \mright] \\
&= T \alpha_{i^*} \beta_{i^*}-\sum_{t=1}^T \mathbb{E}\left[ X_{i(t)}(t) Y_{i(t)}(t)\right] \nonumber  \\
& =\sum_{i \neq i^*}\Delta_i \mathbb{E}\left[n_i(T) \right] \nonumber,
\end{align}
where $\Delta_i \triangleq \alpha_{i^*}\beta_{i^*}-\alpha_{i}\beta_{i}$ denotes the reward gap for portion $i$ and $n_i(T) \triangleq \sum_{t=1}^T \id\{i(t) = i\}$ denotes the number of plays of portion $i$ up to timeslot $T$.

To summarize, adaptive portion selection for wireless interactive panoramic scene delivery under 2/F/B feedback proceeds as follows: In each timeslot $t$,
\begin{enumerate}[label=\textcircled{\arabic*}]
    \item The server decides the delivery portion $i(t)$ for timeslot $t$ according to its history of observations $(\mathbf{X}(\tau), Y_{i(\tau)}(\tau),  i(\tau))_{\tau=1}^{t-1}$. \label{step_one}
    \item The server predicts the user's head pose for timeslot $t$ and sends the content occupying the chosen delivery portion $i(t)$ centered on the user's predicted head pose  to the user over the wireless channel.
    \item The server receives the ACK packet for timeslot $t$ and observes the user's actual head pose and the transmission outcome $Y_{i(t)}(t)$. 
    \item The server calculates the prediction outcome $X_i(t)$ for each portion $i$ from the user's actual head pose (received in the ACK packet) and the server's predicted head pose. 
\end{enumerate}

\subsection{Comparison with 1/B and 2/B/B Feedback}

Using the language of the multi-armed bandit literature, we sometimes call a delivery portion $i \in [N]$ an \textit{``arm"}. For a bandit feedback signal, balancing \textit{exploration} and \textit{exploitation} is crucial. That is, exploring underplayed arms vs. exploiting the current most promising arm. The prior work on portion selection \cite{gupta2022online,chen2020thompson} improved on the 1/B feedback model, where the edge server only receives the throughput $Z_{i(t)}(t)$ as feedback, by considering the 2/B/B feedback model, where the edge server receives both $X_{i(t)}(t)$ and $Y_{i(t)}(t)$ as feedback. While 2/B/B gathers more information per timeslot, both feedback signals are bandit, and are therefore subject to the exploration/exploitation dilemma. The major improvement we introduce with the 2/F/B feedback model is to completely remove the necessity for exploration in one of the feedback signals. Also note that 2/F/B feedback yields $N-1$ more pieces of information compared to 2/B/B feedback.  Intuitively, this is because an estimate of the prediction rate improves in every timeslot whether or not portion $i$ is sent, so no plays have to be spent on it. Under 2/B/B both rates can be estimated only from the portions that were actually sent, so plays must be spent resolving the uncertainty in each of them.

\subsection{Discussion and Limitation}
In this subsection, we discuss the limitations of the assumptions in our system model. Specifically, we assume that $\mathbf{X}(t)$ and $\mathbf{Y}(t)$ are i.i.d.\ over time, and that  $\mathbf{X}(t)$ and $\mathbf{Y}(t)$ are independent. These assumptions, also adopted in \cite{chen2020thompson,gupta2022online}, facilitate analytical tractability. In practice, however, $\mathbf{X}(t)$ may exhibit temporal dependencies due to user head movements, violating the i.i.d.\ assumption, while $\mathbf{Y}(t)$ may deviate from i.i.d.\ behavior due to variations in the panoramic scene content. Furthermore, $\mathbf{X}(t)$ and $\mathbf{Y}(t)$ may be dependent because recent head poses can be related to the scene content.
 The temporal dependence of the transmission outcome is particularly well documented. A long line of measurement and modeling work shows that wireless packet loss is bursty and temporally correlated rather than i.i.d. The classical Gilbert-Elliott model captures this through a two-state Markov chain that alternates between a low-loss and a high-loss state \cite{gilbert1960capacity, elliott1963estimates}, and subsequent studies further model loss using higher-order Markov chains. This correlation is present in our own measurements: the transmission indicator $Y_i(t)$ is observed directly on a live wireless link, and we report its autocorrelation in Section~\ref{section:empirical_evaluation}. We therefore view the i.i.d.\ assumption on $\mathbf{Y}(t)$ as a tractable approximation adopted for the analysis rather than an accurate description of the physical channel.
Recall that we also assumed $\mathbf{Y}(t)$ is component-wise independent. However, in reality, if we successfully deliver the chosen portion, then we know that we could have also successfully delivered a smaller portion, and if we fail to deliver the chosen portion, then we know that we would have also failed to deliver a larger portion.
 Logarithmic regret has been established for particular Markovian bandit models under additional assumptions, with constants that depend on the mixing properties of the underlying chains \cite{anantharam1987markovian, tekin2012rested}. These results are for index policies designed for those models and do not directly establish regret guarantees for AdaPort under temporally dependent feedback. Our guarantees assume the i.i.d.\ model above, and robustness to temporal dependence is evaluated empirically. To accommodate such temporal correlation and potential non-stationarity in practice, a natural variant of AdaPort replaces the empirical-mean and Beta-posterior updates with their sliding-window counterparts, which discount stale observations. We examine this sliding-window variant empirically in Section~\ref{section:empirical_evaluation}. Despite these modeling assumptions, it is worth noting that our real-world trace-based evaluations, which \textit{do not} rely on any i.i.d. or independence assumption, show that the proposed algorithm consistently outperforms baseline methods. AdaPort even performs better than its sliding-window variant SW-AdaPort when the window size is poorly chosen. We note that choosing the window size typically requires more knowledge about the environment. For example, in an abruptly changing environment, knowledge of the number of abrupt changes is required to optimally tune the window size \cite{garivier2008upperconfidenceboundpoliciesnonstationary}. Finally, we note that the 2/F/B feedback model is of independent theoretical interest beyond the 2/B/B model studied in \cite{chen2020thompson,gupta2022online}.

\section{Algorithm Design}
\label{section:algorithm_design}
Recall the step-by-step summary of the sequence of events for the adaptive portion selection problem under 2/F/B feedback presented in Section~\ref{section:system_model}. In this section, we focus on step~\ref{step_one}. Specifically, how to design an adaptive portion selection algorithm to select the delivery portion $i(t)$ in timeslot $t$ given the history of observations $(\mathbf{X}(\tau), Y_{i(\tau)}(\tau),  i(\tau))_{\tau=1}^{t-1}$. Recall that the optimal policy always selects $i^* \in \argmax_{i\in[N]} \alpha_i \beta_i$ in each timeslot $t$. Since $(\alpha_i, \beta_i)_{i=1}^N$ are unknown, we instead choose 
\begin{equation}
\label{alg:policy}
i(t) \in \argmax_{i\in[N]}\, \overline{\alpha}_i(t)\, \theta_{\beta,i}(t).
\end{equation}
where each $\overline{\alpha}_i(t)$ and $\theta_{\beta,i}(t)$ are our best estimates of $\alpha_i$ and $\beta_i$ respectively, estimated from the history of observations. The details of how these estimates are chosen are given in the following subsections, and our algorithm, AdaPort is given in Algorithm~\ref{alg:mixed_feedback}.

\subsection{Prediction Rate Estimate}
Recall that the prediction outcome is a full-information feedback signal. Then to estimate $\alpha_i$ for each portion $i \in [N]$, we simply use the empirical mean $\overline{\alpha}_i(t) \triangleq \frac{1}{t-1}\sum_{\tau=1}^{t-1} X_i(\tau)$ for timeslots $t > 1$ and $\overline{\alpha}_i(1) = 0$ (the value of $\overline{\alpha}_i(1)$ is arbitrary under full-information). This estimate can be updated iteratively using the recurrence 
\begin{equation}
\label{eq:prediction_update}
\overline{\alpha}_i(t+1) = \overline{\alpha}_i(t) + \frac{1}{t}\left(X_i(t) - \overline{\alpha}_i(t)\right).
\end{equation}

\subsection{Transmission Rate Estimate}

Unlike the prediction outcome signal, the transmission outcome is a bandit feedback signal. Therefore, we need to use an estimate of each $\beta_i$ that is able to balance exploration and exploitation. Following \cite{chen2020thompson}, we adopt a Bayesian approach based on Thompson Sampling to estimate the transmission rate. Under Thompson Sampling, the estimate $\theta_{\beta,i}(t)$ is drawn from a posterior distribution representing our current belief about the unknown parameter $\beta_i$. For Bernoulli rewards, we draw the Thompson sample $\theta_{\beta,i}(t) \sim \Beta(S_{\beta,i}(t) + 1, F_{\beta,i}(t) + 1)$, where $S_{\beta,i}(t)$ and $F_{\beta,i}(t)$ denote the number of observed transmission successes and failures respectively. In each timeslot the samples are drawn independently across portions, using randomness independent of the feedback, and ties in \eqref{alg:policy} are broken by any fixed rule. These are initialized as $S_{\beta,i}(1) = 0$ and $F_{\beta,i}(1) = 0$ and updated as
\begin{equation}
\label{eq:transmission_update}
\begin{aligned}
S_{\beta,i}(t + 1) &= S_{\beta,i}(t) + \id\{i(t) = i\}\,Y_{i(t)}(t), \text{ and} \\
F_{\beta,i}(t + 1) &= F_{\beta,i}(t) + \id\{i(t) = i\}\left(1-Y_{i(t)}(t) \right). \\
\end{aligned}
\end{equation}
Exploration of each portion $i$ is accomplished in Thompson sampling due to the variance in the Beta distribution, which shrinks as the number of plays $n_i(t)$ of portion $i$ increases.

 While the empirical mean and the Thompson sample are each individually standard, the novelty of AdaPort lies in how it exploits the hybrid feedback structure: the full-information empirical mean $\overline{\alpha}_i(t)$ and the bandit Thompson sample $\theta_{\beta,i}(t)$ are combined through the product selection rule \eqref{alg:policy}, which couples the two feedback levels. This coupling makes the regret analysis nontrivial, because a suboptimal arm is selected when the random product $\overline{\alpha}_i(t)\,\theta_{\beta,i}(t)$ exceeds a threshold, rather than when a single posterior sample exceeds a fixed mean. Handling this random, full-information-dependent threshold is the main departure from the standard Thompson Sampling analysis, as we detail in the proof of Theorem~\ref{theorem:upper_bound}.

\begin{algorithm}
\caption{Hybrid Feedback-Driven Learning for Adaptive Portion Selection Algorithm (AdaPort)}\label{alg:mixed_feedback}
\begin{algorithmic}[1]  
    \STATE \textbf{Initialization:} Set $S_{\beta,i}(1) = F_{\beta,i}(1) =  \overline{\alpha}_i(1) = 0 \quad\forall\,\,i \in [N]$.
    \FOR{each $t=1,2,\ldots$}
        \FOR{each portion $i$}
            \STATE Draw $
            \displaystyle\theta_{\beta,i}(t) \sim \text{Beta}\left(S_{\beta,i}(t)+1, F_{\beta,i}(t)+1\right).
            $
        \ENDFOR
        \STATE Send the delivery portion $\displaystyle i(t) \in \arg\max_{i} \overline{\alpha}_i(t)\, \theta_{\beta,i}(t)$.
        \STATE Receive the feedback $(\mathbf{X}(t), Y_{i(t)}(t))$. \footnotemark
        \FOR{each portion $i$}
            \STATE Update $\overline{\alpha}_i(t)$ according to \eqref{eq:prediction_update}.
            \STATE Update $S_{\beta,i}(t)$ and $F_{\beta,i}(t)$ according to \eqref{eq:transmission_update}.
        \ENDFOR
    \ENDFOR
\end{algorithmic}
\end{algorithm}
\footnotetext{Refer to steps (3) and (4) in the step-by-step summary of events in Section~\ref{section:system_model}.}

We conclude this section by examining the computational overhead of AdaPort. In each timeslot, AdaPort draws one Beta sample $\theta_{\beta,i}(t)$ per portion, selects the maximizer of the product $\overline{\alpha}_i(t)\,\theta_{\beta,i}(t)$, and then performs a constant-time update of the empirical prediction mean $\overline{\alpha}_i(t)$ via \eqref{eq:prediction_update} and of the success/failure counts $S_{\beta,i}(t),F_{\beta,i}(t)$ via \eqref{eq:transmission_update}. Drawing a Beta sample and updating a Beta posterior are both $O(1)$ operations, so the per-timeslot time and memory complexity of AdaPort is $O(N)$, linear in the number of candidate portions $N$ with a small constant and no dependence on the time horizon $T$. The per-timeslot cost therefore grows linearly in $N$ and stays small for the dozens of candidate portions that arise in high-resolution (e.g., 8K) tiled delivery. It consists of $N$ Beta draws and $2N$ counter updates, none of which involves the tile data itself, so it is small relative to the rendering and transmission cost incurred in each frame interval.

\section{Performance Analysis}
\label{section:performance_analysis}

In this section, we begin by deriving the regret lower bound for the online learning problem with 2/F/B feedback. We then compare this lower bound against the known lower bounds for 2/B/B feedback derived by Gupta et al. \cite{gupta2022online} and for 1/B feedback derived by Lai and Robbins \cite{lai1985asymptotically}. To further support our motivation, we present simple illustrative examples that highlight the significant improvement in the lower bound under the 2/F/B feedback setting. Finally, we establish the regret upper bound of AdaPort (Algorithm~\ref{alg:mixed_feedback}), which matches the lower bound asymptotically. 




\subsection{Regret Lower Bound under 2/F/B Feedback}


In the following theorem, we present our asymptotic instance-dependent regret lower bound for online learning under 2/F/B feedback. 
\begin{theorem}
\label{theorem:lower_bound}
(Lower bound) Consider an online learning algorithm under 2/F/B feedback that is \emph{consistent} \cite[Definition 16.1]{lattimore2020bandit}, i.e., on every problem instance of the model in Section~\ref{section:system_model}, its regret satisfies $R(T) = o(T^\delta)$ for every $\delta > 0$. Then the algorithm is subject to the following regret lower bound:
\begin{equation}
\label{eq:one_full_one_bandit_lower_bound}
\liminf_{T \rightarrow \infty} \frac{R(T)}{\log T} \geq \sum_{i \neq i^*\,:\,\,\alpha_{i}>\alpha_{i^*}\beta_{i^*}} \frac{\Delta_i}{ d\mleft(\beta_i, \frac{\alpha_{i^*}\beta_{i^*}}{\alpha_i}\mright)}.
\end{equation}
\end{theorem}

\begin{proof}
Recall from \eqref{system_model_regret} that the regret can be decomposed as $R(T) = \sum_{i \neq i^*}\Delta_i \mathbb{E}\left[n_i(T) \right]$. Then it suffices to show that $
\liminf_{T\to\infty}\frac{\bE\left[n_i(T)\right]}{\log T}
\ge \frac{1}{d\!\left(\beta_i,\frac{\alpha_{i^*}\beta_{i^*}}{\alpha_i}\right)}, \forall\, i \in \mathcal{I}$,
where
$
\mathcal{I} \triangleq \{\, i\neq i^* : \alpha_i > \alpha_{i^*}\beta_{i^*} \,\}$.
Given arm $j \in \mathcal{I}$, fix $\lambda \in (\Delta_j, \alpha_j(1-\beta_j))$. We construct a new problem instance in which the distribution of $\mathbf{X}(t)$ is unchanged and the transmission rates are $(\beta_i')_{i=1}^N$ with $\beta_j' = \beta_j + \lambda / \alpha_j$, and $\beta_i' = \beta_i$ for all $i \neq j$. The remainder of the proof is similar to the usual instance-dependent lower bound for stochastic bandits (see e.g. Theorem 16.2 in \cite{lattimore2020bandit}), with some slight modifications to account for the two levels of feedback. The full proof can be found in Appendix~\ref{appendix:lower_bound}.
\end{proof}
The right-hand side of \eqref{eq:one_full_one_bandit_lower_bound} is called the \textit{``lower bound constant"}. A key observation about the bound in Theorem~\ref{theorem:lower_bound} is that arms with $\alpha_i \leq \alpha_{i^*}\beta_{i^*}$ have zero contribution to the lower bound constant. This shows that
an optimal algorithm (i.e., one with matching upper bound) in this setting plays these arms only $o(\log T)$ times in expectation,
whereas optimal algorithms in the 2/B/B and 1/B feedback regimes require a continual $\log T$ number of plays of each suboptimal arm when $\alpha_{i^*}\beta_{i^*} < 1$. For AdaPort, the upper bound analysis in Section~\ref{subsection:upper_bound} further shows that each arm with $\alpha_i < \alpha_{i^*}\beta_{i^*}$ is played a bounded expected number of times; at the boundary $\alpha_i = \alpha_{i^*}\beta_{i^*}$, the analysis gives $o(\log T)$ expected plays.
Intuitively, if $\alpha_i \leq \alpha_{i^*} \beta_{i^*}$, since the full-information feedback structure ensures the rapid convergence of the estimation of the prediction rate, the uncertainty is governed by the estimation of the transmission rate. In this regime, suboptimal arm $i$ cannot outperform the optimal arm even if its transmission rate is maximally overestimated to $1$. Consequently, $O(\log T)$ exploration is not required to resolve the uncertainty of the estimation of transmission rate, as the selection policy can tolerate such overestimation without compromising optimality.


\begin{remark}[lower bound constant comparison]
\label{remark:size_comparison}

\begin{equation}
\begin{aligned}
&\overbrace{\sum_{\substack{i \neq i^*\,: \\\alpha_{i}>\alpha_{i^*}\beta_{i^*}}} \frac{\Delta_i}{ d\mleft(\beta_i, \frac{\alpha_{i^*}\beta_{i^*}}{\alpha_i}\mright)}}^{\text{2/F/B lower bound constant}} \\
&= \sum_{\substack{i \neq i^*\,: \\\alpha_{i}>\alpha_{i^*}\beta_{i^*}}} \frac{\Delta_i}{ d(\alpha_i, x_i) + d(\beta_i, y_i)}, \quad \begin{aligned} &x_i = \alpha_i,\\ & y_i = \frac{\alpha_{i^*}\beta_{i^*}}{\alpha_i} \end{aligned} \\
&\leq \sum_{\substack{i \neq i^*\,: \\\alpha_{i}>\alpha_{i^*}\beta_{i^*}}} \frac{\Delta_i}{ \min_{\substack{0 \leq x_i, y_i \leq 1 \\ x_i y_i \geq \alpha_{i^*} \beta_{i^*}}} d\left(\alpha_i, x_i\right)+d\left(\beta_i, y_i\right)} \\
&\overset{(i)}{\leq} \underbrace{\sum_{i \neq i^*} \frac{\Delta_i}{ \min_{\substack{0 \leq x_i, y_i \leq 1 \\ x_i y_i \geq \alpha_{i^*} \beta_{i^*}}} d\left(\alpha_i, x_i\right)+d\left(\beta_i, y_i\right)}}_{\text{2/B/B lower bound constant derived by \cite{gupta2022online}}} \\
&\overset{(ii)}{\leq} \underbrace{\sum_{i \neq i^*} \frac{\Delta_i}{ \min_{\substack{0 \leq x_i, y_i \leq 1 \\ x_i y_i \geq \alpha_{i^*} \beta_{i^*}}} d\left(\alpha_i\beta_i, x_i y_i\right)}}_{\text{1/B lower bound constant derived by \cite{lai1985asymptotically}}}
\end{aligned}
\end{equation}

The above shows that the lower bound constant under 2/F/B given in Theorem~\ref{theorem:lower_bound} is no larger than the lower bound constant under 2/B/B derived by Gupta et al. \cite{gupta2022online}, and strictly smaller on some instances (Figure~\ref{fig:lower_bounds_combined}).
This follows from the unlabelled inequality, in which the lower bound constant under 2/B/B minimizes the denominator over $(x_i, y_i)$ for the suboptimal arms satisfying $\alpha_i>\alpha_{i^*}\beta_{i^*}$, unlike the lower bound constant under 2/F/B, and from the relationship $(i)$, which adds the nonnegative terms of the arms with $\alpha_i \leq \alpha_{i^*}\beta_{i^*}$. The relationship $(ii)$ between the lower bound constants under 2/B/B feedback and 1/B feedback is proven in Theorem 2 of \cite{gupta2022online} based on the data processing inequality.
\end{remark}

We further illustrate the size comparison between the lower bound constants under 2/F/B, 2/B/B, and 1/B feedback shown in Remark~\ref{remark:size_comparison} by numerical simulations with two arms.  
Specially, we set the optimal arm's prediction and transmission rates as $\alpha_{i^*}=0.8$ and $\beta_{i^*}=0.9$, respectively. In the first example (Figure~\ref{fig:lowerbound1}), we fix the suboptimal arm’s transmission rate at 0.75 and vary its prediction rate across the set 
$[0.75, 0.8, 0.85,0.9]$. In the second example (Figure~\ref{fig:lowerbound2}), we fix the prediction rate at 0.75 and vary the transmission rate across $[0.6,0.7,0.8,0.9]$.
In Figure~\ref{fig:lower_bounds_combined}, we observe that the lower bound constants corresponding to 1/B and 2/B/B settings exhibit small differences. In contrast, the lower bound constant of 2/F/B is significantly smaller, indicating that a learning algorithm may substantially reduce its regret by exploiting the additional full-information feedback. 

\begin{figure}[ht]
    \centering
    \begin{subfigure}[b]{0.23\textwidth}
        \centering
        \includegraphics[width=\textwidth, height=0.17\textheight]{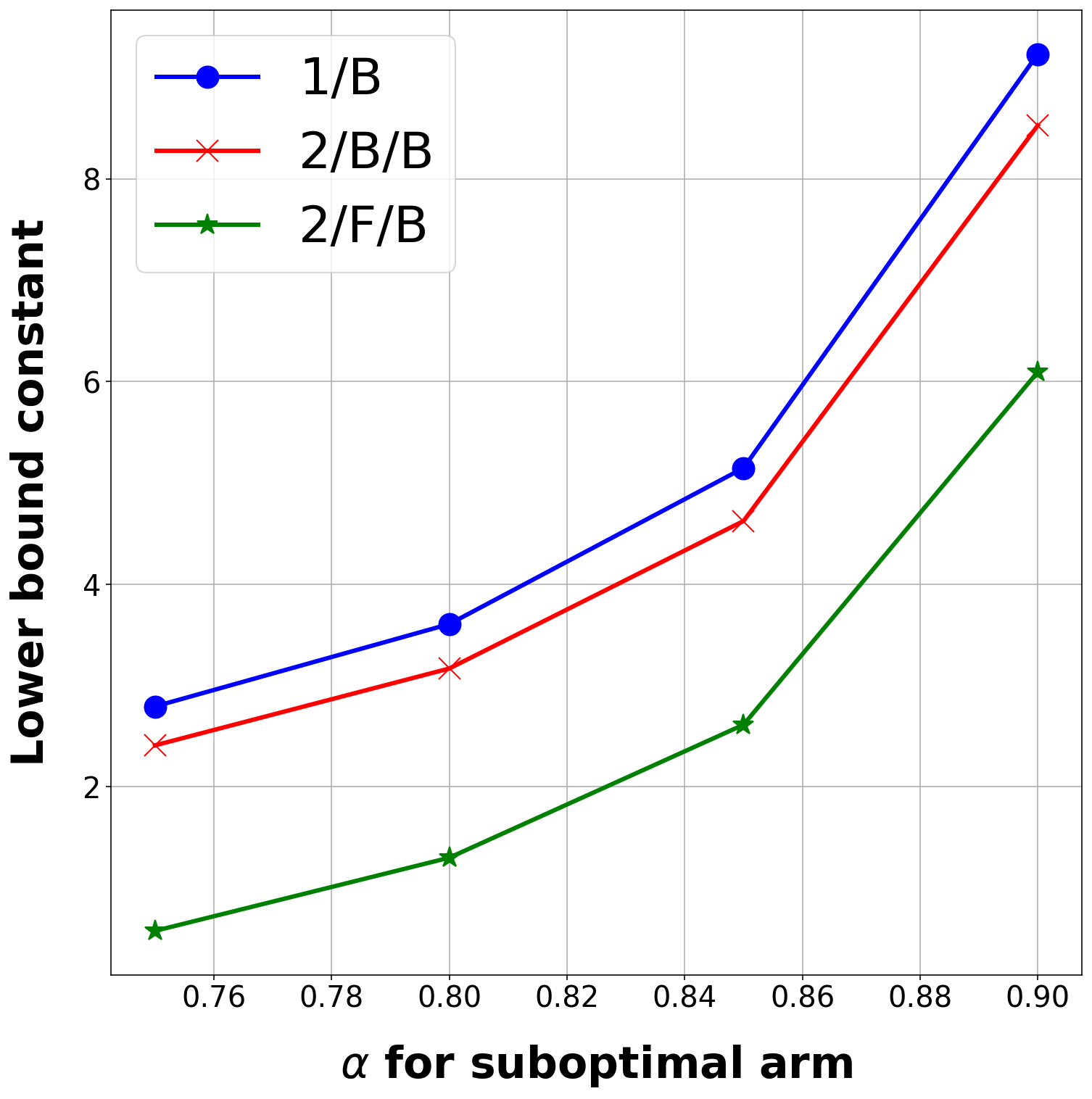}
        \caption{Fixing $\beta$ and varying $\alpha$.}
        \label{fig:lowerbound1}
    \end{subfigure}
    \hfill
    \begin{subfigure}[b]{0.23\textwidth}
        \centering
        \includegraphics[width=\textwidth, height=0.17\textheight]{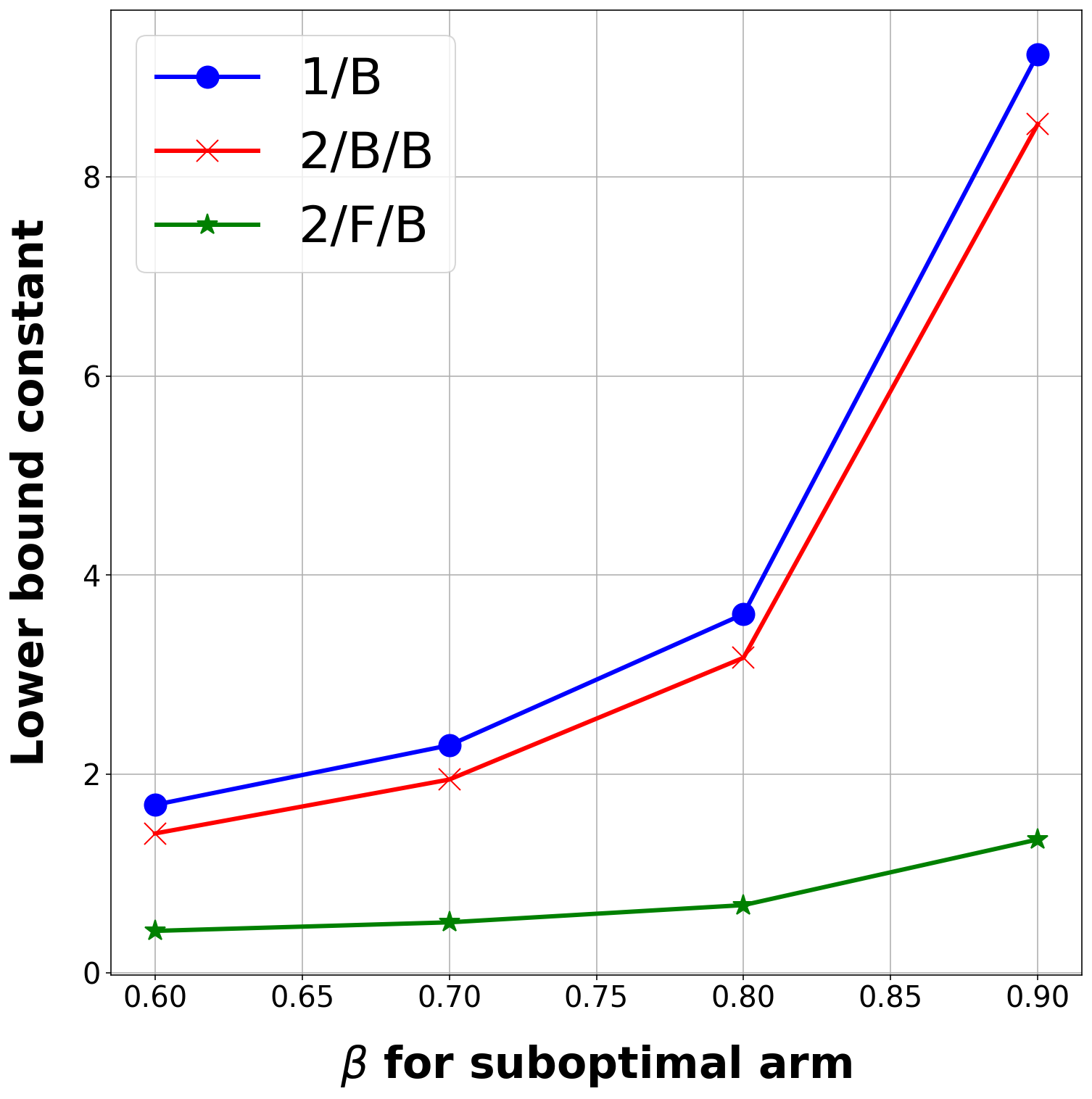}
        \caption{Fixing $\alpha$ and varying $\beta$.}
        \label{fig:lowerbound2}
    \end{subfigure}
    \caption{Lower bound constant comparison for two arms.}
    \label{fig:lower_bounds_combined}
\end{figure}

\subsection{Regret Upper Bound of AdaPort}
\label{subsection:upper_bound}

In the following theorem, we validate that AdaPort (Algorithm~\ref{alg:mixed_feedback}) achieves an upper bound constant equal to the lower bound constant in Theorem \ref{theorem:lower_bound}, demonstrating its asymptotic optimality.
\begin{theorem}
\label{theorem:upper_bound}
(Upper bound) AdaPort (Algorithm~\ref{alg:mixed_feedback}) achieves an asymptotic regret upper bound of 
\begin{equation}
\label{eqn:upper_bound_result}
\limsup _{T \rightarrow \infty} \frac{R(T)}{\log T} \leq \sum_{i \neq i^*\,:\,\,\alpha_{i}>\alpha_{i^*}\beta_{i^*}} \frac{\Delta_i}{ d\mleft(\beta_i, \frac{\alpha_{i^*}\beta_{i^*}}{\alpha_i}\mright)}.    
\end{equation}
\end{theorem}
\begin{proof}
Recall from \eqref{system_model_regret} that the regret can be decomposed as $R(T) = \sum_{i \neq i^*}\Delta_i \mathbb{E}\left[n_i(T) \right]$. Fix an arbitrary suboptimal arm $i \neq i^*$. Then it suffices to show that
\begin{equation}
\label{eq:suffices_to_show}
\limsup_{T\rightarrow\infty} \frac{\bE\left[n_i(T)\right]}{\log T} \leq \begin{cases}
\frac{1}{d(\beta_i,\frac{\alpha_{i^*}\beta_{i^*}}{\alpha_i})} &     \alpha_{i} > \alpha_{i^*}\beta_{i^*} \\
0 & \alpha_{i} \leq \alpha_{i^*}\beta_{i^*}
\end{cases}.
\end{equation}
We begin by selecting $\epsilon_1, \epsilon_2, \epsilon_3 > 0$ that satisfy
\begin{equation}
\label{eqn:epsilon_range}
   (\alpha_i + \epsilon_1)(\beta_i + \epsilon_1) < \alpha_{i^*} \beta_{i^*} - \epsilon_3 < (\alpha_{i^*} - \epsilon_2) \beta_{i^*}, 
\end{equation}
where the first inequality is applied in deriving \eqref{decompose_second_term}, while the second inequality is applied in deriving \eqref{decompose_first_term}.

In particular, fix $\epsilon_3 \in (0, \epsilon_3^{\max})$ where 
\begin{equation*}
\epsilon_3^{\max} \triangleq \begin{cases}
\Delta_i &  \alpha_i \geq \alpha_{i^*} \beta_{i^*} \\
\alpha_{i^*} \beta_{i^*} - \alpha_i & \alpha_i < \alpha_{i^*} \beta_{i^*}.
\end{cases}.
\end{equation*}
\begin{figure}[htbp]
  \centering
  \includegraphics[width=0.35\textwidth, height=0.15\textwidth]{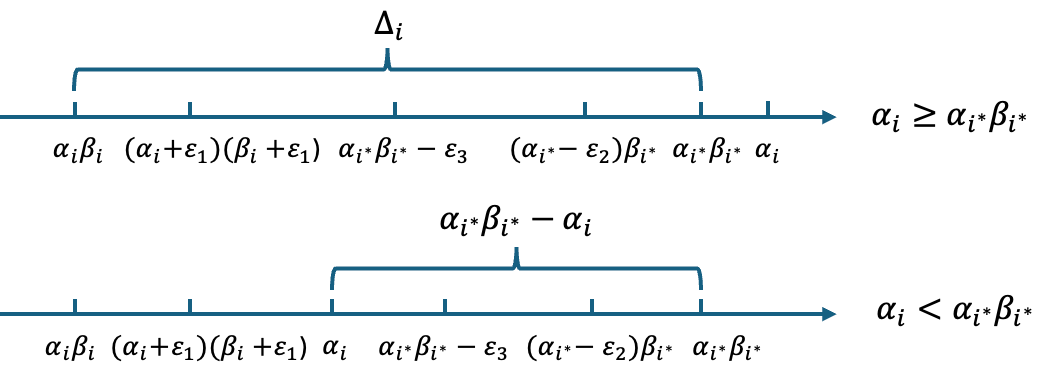}
  \caption{The range of $\epsilon_3$.}
  \label{fig:epsilon_range}
\end{figure}

Then define the intervals $\Gamma_1 \triangleq \left(0,\epsilon_1^{\max}\right)$  and $\Gamma_2 \triangleq \left(0,\ \frac{ \epsilon_3 }{ \beta_{i^*} } \right)$ where $\epsilon_1^{\max} \triangleq \frac{ \sqrt{(\alpha_i + \beta_i)^2 + 4(\Delta_i - \epsilon_3)} - (\alpha_i + \beta_i)}{2}$ and fix $\epsilon_1 \in \Gamma_1$ and $\epsilon_2 \in \Gamma_2$. Note that $\epsilon_1^{\max}$ is a zero of the quadratic function $q_i(\epsilon_1) = \epsilon_1^2 + (\alpha_i + \beta_i)\,\epsilon_1 - \bigl(\Delta_i - \epsilon_3\bigr)$. Then since $\epsilon_1 < \epsilon_1^{\max}$, we have $q_i(\epsilon_1) < 0$ due to the convexity of $q_i(\epsilon_1)$. Rearranging gives $\epsilon_1^2 + (\alpha_i + \beta_i)\,\epsilon_1 < \Delta_i - \epsilon_3$, and adding $\alpha_i \beta_i$ to both sides  and factoring the left-hand side gives $(\alpha_i + \epsilon_1)(\beta_i + \epsilon_1) < \alpha_{i^*} \beta_{i^*} - \epsilon_3$. Combining this lower bound with the upper bound $\alpha_{i^*} \beta_{i^*} - \epsilon_3 < (\alpha_{i^*} - \epsilon_2) \beta_{i^*}$, which follows directly from the definition of $\Gamma_2$ and the fact that $\epsilon_2 \in \Gamma_2$, gives the desired \eqref{eqn:epsilon_range}.

Next, we use $\epsilon_1, \epsilon_2, \epsilon_3$ to define the following three events, which will be used to decompose $\mathbb{E}\left[n_i(T) \right]$:
\begin{definition}
\label{def:three_events}
(Events $ E_i^\mu(t) $, $ G_{i^*}^\mu(t) $, and $ E_i^\theta(t) $)

For each time step $t$, we define the following events:
\begin{enumerate}
\item 
$ E_i^\mu(t) \triangleq \left\{ \left| \overline{\alpha}_i(t) - \alpha_i \right| \leq \epsilon_1,\ \left| \overline{\beta}_i(t) - \beta_i \right| \leq \epsilon_1 \right\},$
where
$$ 
\overline{\beta}_i(t) = 
\begin{cases}
\frac{1}{n_i(t-1)} \sum_{s=1}^{n_i(t-1)} Y_i(\tau_{s,i}), & \text{if } n_i(t-1) > 0 \\
0, & \text{otherwise}
\end{cases},
 $$
 and $\tau_{s,i}$ denotes the timeslot that arm $i$ was played for the $s$-th time.
 
\item $
G_{i^*}^\mu(t) \triangleq \left\{ \left| \overline{\alpha}_{i^*}(t) - \alpha_{i^*} \right| \leq \epsilon_2 \right\}.
$

\item $
E_i^\theta(t) \triangleq \left\{ \overline{\alpha}_i(t)\, \theta_{\beta,i}(t) \leq \alpha_{i^*} \beta_{i^*} - \epsilon_3 \right\}.
$
\end{enumerate}
\end{definition}
Consider the following four events composed from the events from Definition~\ref{def:three_events}: 

(i) $\{i(t)=i\} \cap E_i^{\mu}(t)\cap E_i^\theta(t)\cap G_{i^*}^\mu(t)$, 

(ii) $\{i(t)=i\} \cap E_i^{\mu}(t)\cap E_i^\theta(t)\cap \overline{G_{i^*}^\mu(t)}$, 

(iii) $\{i(t)=i\} \cap E_i^{\mu}(t)\cap \overline{E_i^\theta(t)}$, and 

(iv) $\{i(t)=i\} \cap \overline{E_i^{\mu}(t)}$. 

These events are mutually exclusive and collectively exhaustive, and moreover, for event (ii), we have
$
\{i(t)=i\} \cap E_i^{\mu}(t)\cap E_i^\theta(t)\cap \overline{G_{i^*}^\mu(t)} 
\subseteq \overline{G_{i^*}^\mu(t)}$.
By the monotonicity of probability and the law of total probability, we therefore bound
\begin{equation}
\label{eqn:initial_decomposition}
\begin{aligned}
 \mathbb{E}\left[n_i(T) \right]
&\textstyle = \sum_{t=1}^T\Pr(i(t) = i) \\
&\leq \textstyle \sum_{t=1}^T\Pr \left(i(t)=i, E_i^\mu(t), E_i^\theta(t),  G_{i^*}^\mu(t)\right) \\
&\phantom{\,=\,}+ \textstyle\sum_{t=1}^T \Pr \left(\overline{G_{i^*}^\mu(t)}\right) \\
&\phantom{\,=\,}+ \textstyle \sum_{t=1}^T\Pr\left(i(t)=i, E_i^\mu(t), \overline{E_i^\theta(t)}\right) \\
&\phantom{\,=\,}+  \textstyle \sum_{t=1}^T\Pr\left(i(t)=i, \ol{E_i^\mu(t)}\right).
\end{aligned}
\end{equation}

The following lemma, whose proof can be found in Appendix~\ref{appendix:upper_bound} and relies on the relationship \eqref{eqn:epsilon_range} between $\epsilon_1,\epsilon_2,\epsilon_3$, bounds each of the terms above:
\begin{lemma}
\label{lemma:upper_bound}
For the fixed $\epsilon_1, \epsilon_2, \epsilon_3$ chosen above, each of the terms in \eqref{eqn:initial_decomposition} can be bounded by
\begin{equation}
\label{decompose_first_term}
\textstyle \sum_{t=1}^T \Pr \left(i(t)=i, E_i^\mu(t), E_i^\theta(t),  G_{i^*}^\mu(t)\right)
\leq M_1 ,
\end{equation}
\begin{align}
\label{decompose_second_term}
& \textstyle \sum_{t=1}^T \Pr\left(i(t)=i, E_i^\mu(t), \overline{E_i^\theta(t)}\right) \nonumber\\
& \leq
\begin{cases}
\frac{\log T}{d(\beta_i+\epsilon_1,\frac{\alpha_{i^*}\beta_{i^*}-\epsilon_3}{\alpha_i+\epsilon_1})}+2 &  \alpha_i \geq \alpha_{i^*} \beta_{i^*} \\
M_2 & \alpha_i < \alpha_{i^*} \beta_{i^*}
\end{cases},
\end{align}
\begin{equation}
\label{decompose_third_term}
\textstyle  \sum_{t=1}^T \Pr\left(i(t)=i, \ol{E_i^\mu(t)}\right) 
\leq M_3,
\end{equation}
\begin{equation}
\label{decompose_fourth_term}
\textstyle \sum_{t=1}^T \Pr \left(\overline{G_{i^*}^\mu(t)}\right) 
\leq M_4,
\end{equation}
where $M_1\triangleq C_{\mathrm{post}}(\beta_{i^*}, x)$ is the finite constant given in Lemma~\ref{lem:posterior_sum} in Appendix~\ref{appendix:upper_bound}, $x \triangleq \frac{\alpha_{i^*} \beta_{i^*}-\epsilon_3}{\alpha_{i^*}-\epsilon_2}$, $M_2\triangleq 1+\frac{1}{2(\alpha_{i^*} \beta_{i^*}-\alpha_i-\epsilon_3)^2}$ for $\alpha_i < \alpha_{i^*}\beta_{i^*}$,$ M_3\triangleq \frac{2}{\epsilon_1^2}+1$, $M_4\triangleq \frac{1}{\epsilon_2^2}+1$.
\end{lemma}
From \eqref{eqn:initial_decomposition} and Lemma~\ref{lemma:upper_bound}, it follows that if $\alpha_i < \alpha_{i^*}\beta_{i^*}$, then 
\begin{equation}
\mathbb{E}[n_i(T)] \leq  M_1 + M_2 + M_3 + M_4.
\end{equation}
In this case, since $M_1,M_2, M_3, M_4$ are constants, it follows that $\limsup_{T\rightarrow\infty} \frac{\bE\left[n_i(T)\right]}{\log T} =0$, which agrees with \eqref{eq:suffices_to_show}. On the other hand, if $\alpha_i \geq \alpha_{i^*}\beta_{i^*}$, then \eqref{eqn:initial_decomposition} and Lemma~\ref{lemma:upper_bound} give that
\begin{equation}
\mathbb{E}[n_i(T)] \leq
      \frac{\log T}{
        d\!\left(\beta_i+\epsilon_1,\,
        \tfrac{\alpha_{i^*}\beta_{i^*}-\epsilon_3}{\alpha_i+\epsilon_1}\right)} + M_1 + M_3 + M_4+2,
\end{equation}
and therefore
\begin{equation}
\begin{aligned}
\limsup_{T\rightarrow\infty} \frac{\bE\left[n_i(T)\right]}{\log T}
&\leq \frac{1}{d(\beta_i+\epsilon_1,\frac{\alpha_{i^*}\beta_{i^*}-\epsilon_3}{\alpha_i+\epsilon_1})}. \\
\end{aligned}
\end{equation}
Since this holds for every $\epsilon_1 \in \Gamma_1$, letting $\epsilon_1 \downarrow 0$ gives
$\limsup_{T\rightarrow\infty} \frac{\bE\left[n_i(T)\right]}{\log T} \leq \frac{1}{d(\beta_i,\frac{\alpha_{i^*}\beta_{i^*}-\epsilon_3}{\alpha_i})}$.
Since $\epsilon_3$ can be taken arbitrarily small and positive, we conclude that
$\limsup_{T\rightarrow\infty} \frac{\bE\left[n_i(T)\right]}{\log T} \leq \frac{1}{d(\beta_i,\frac{\alpha_{i^*}\beta_{i^*}}{\alpha_i})}$, which agrees with \eqref{eq:suffices_to_show}. In the boundary case $\alpha_i = \alpha_{i^*}\beta_{i^*}$, we have $\beta_i < 1$ since $\Delta_i > 0$, so the right-hand side is $1/d(\beta_i, 1) = 0$ and $\bE[n_i(T)] = o(\log T)$.
\end{proof}

\begin{remark}
According to Lemma~\ref{lemma:upper_bound}, the regret contribution from 
$\sum_{t=1}^T \Pr\left(i(t)=i, E_i^\mu(t), \overline{E_i^\theta(t)}\right)$
is of order $O(\log T)$. Under the 2/B/B feedback model, even if the empirical estimates $\overline{\alpha}_i(t)$ and $\overline{\beta}_i(t)$ are close to their true means $\alpha_i$ and $\beta_i$, the arm selection remains uncertain due to the stochasticity of both $\theta_{\alpha,i}(t)$ and $\theta_{\beta,i}(t)$. These variables represent Thompson Sampling draws for the prediction and transmission rates, respectively, and are updated based on bandit feedback \cite{chen2020thompson, gupta2022online}. Thus, uncertainty persists in both dimensions.
In contrast, under the 2/F/B feedback model, the prediction rate $\alpha_i$ is estimated with full information, so $\overline{\alpha}_i(t)$ concentrates rapidly around $\alpha_i$. As a result, the only remaining source of exploration-induced randomness is $\theta_{\beta,i}(t)$, which governs the uncertainty in the transmission rate. This reduction in uncertainty leads to more confident decision-making and lower regret in practice.
\end{remark}

\section{Empirical Evaluation}
\label{section:empirical_evaluation}

In this section, we evaluate AdaPort on data measured from an end-to-end panoramic scene delivery testbed, in which the head motion of a user interacting with a panoramic video stream and the wireless transmission outcome are produced by the same live session. We compare the performance of our algorithm against four baseline algorithms and a windowed variant of AdaPort:
\begin{enumerate}
    \item \textbf{Thompson sampling under 1/B feedback} (1/B-TS): the standard Thompson sampling MAB algorithm.
    \item \textbf{Thompson sampling under 2/B/B feedback} (2/B/B-TS): two-level Thompson sampling  (Chen et al. \cite{chen2020thompson}).
    \item \textbf{EXP3 under 1/B feedback} (1/B-EXP3): an optimal algorithm for adversarial bandits (Auer et al. \cite{auer2002nonstochastic}).
    \item \textbf{Heuristic minimum scene
delivery} (Heuristic):  a heuristic method based on the Flare system \cite{qian2018flare} that delivers only those tiles overlapping with the user's predicted viewport.
    \item \textbf{Sliding-window variant} (SW-AdaPort): AdaPort with every estimate formed from the $\tau$ most recent timeslots rather than from the whole history, so that outdated observations are discarded rather than weighted equally with recent ones. The prediction mean is taken over all $\tau$ slots in the window, since the prediction outcome is full information, while the transmission successes and failures are accumulated only over the slots in the window in which the portion was actually sent.
\end{enumerate}
(1) and (2) are state-of-the-art algorithms for the portion selection problem, while (5) is chosen to evaluate the impact of treating the prediction and transmission outcomes as i.i.d.\ in a practical system, since it retains the selection rule of AdaPort and changes only how much history the estimates are formed from. In the following subsections, we detail our findings for the testbed-based evaluation. In order to evaluate user experience, we also plot the \textit{relative throughput degradation} compared to the optimal policy, which is given by:
\begin{equation}
\label{eq:relative_degradation}
\bE\mleft[ \frac{\sum_{t=1}^T (1-Z_{i(t)}(t)) - \sum_{t=1}^T (1-Z_{i^*}(t))}{\sum_{t=1}^T (1-Z_{i^*}(t)) } \mright].
\end{equation}
This is the relative increase in the number of failed deliveries over that of the optimal policy. The optimal policy fails at least once on every recorded episode, so the denominator is positive.

\subsection{Testbed and Experimental Setup}

\subsubsection{Testbed}

The edge server is a workstation attached to an 802.11ac access point by
Gigabit Ethernet, so that the client downlink is the only wireless hop under
study. It holds the tiled panoramic content in memory and runs the portion
selection and streaming logic in Java. The client is a Google Pixel~5 running
Android, associated with the access point on the 5\,GHz band. The measurements
reported here were taken at a received signal strength between $-25$ and
$-31$\,dBm.

Each timeslot proceeds as follows. The server predicts the head pose, selects a
portion, and transmits its tiles as RTP over UDP using a custom $40$-byte header
that carries the frame index, the tile index, and the position of the packet
within its tile. The client reassembles the tiles and returns an
acknowledgment carrying its actual pose for that timeslot. The system runs at
$60$\,FPS and a portion must be delivered within the client's $15$\,ms slot
deadline, with no retransmission attempted, consistent with the timing
constraint discussed in Section~I. Frame emission is closed-loop: the server sends the next frame on
receipt of the acknowledgment for the previous one, so that the offered load
adapts to the link instead of being fixed in advance.

The client is run as a pure network endpoint, with decoding and rendering
disabled. With the full decode pipeline active the client runs fifteen
concurrent hardware decoders and reaches severe thermal throttling within
minutes. The receive path is then delayed and slots are recorded as transmission
failures even when their packets arrived on time. The effect grows with portion
size, since a wider portion carries more tiles to decode, and therefore
introduces a bias in $\beta_i$ that increases with portion size, which is
precisely the quantity being compared. With decoding disabled, the client
remained at its idle temperature, and no thermal throttling was observed in any
of the sessions reported here.

\subsubsection{Delivery portions}
\label{trace_portion_setup}
We collect a data trace from a user viewing a free educational panoramic video (see \cite{spacepanovideo}).
The trace captures a 3 DoF (orientation only) head motion over $T = 3\times 10^3$ timeslots. At each timeslot $t$, we apply a linear regression model to the motion trace in order to predict the user's head pose for the subsequent timeslot $t+1$. The model is trained using motion trajectories from the preceding three timeslots. We set $N=4$ different portions to select from: $100\degree\times90\degree$ (minimum viewport), $102\degree\times91\degree$, $108\degree\times94\degree$, and $120\degree\times100\degree$, where each pair gives the angles corresponding to the yaw and pitch axes, respectively. For each portion $i\in [N]$, we compute $X_i(t)$ from the user's actual head pose (received in the ACK packet) and the server's predicted head pose, as described in step (4) in Section \ref{section:system_model}.

\subsubsection{Trace collection and evaluation protocol}

Both feedback signals are obtained from the same live session. The transmission
outcome $Y_{i(t)}(t)$ is reported by the client and equals $1$ only if every
packet of the transmitted portion arrived before that timeslot's deadline, and a
single lost or late packet makes it $0$. Once the actual pose arrives in the
acknowledgment, $X_i(t)$ is recomputed for \emph{every} portion $i \in [N]$
rather than only for the one that was sent. This is why the prediction outcome
constitutes full-information feedback rather than bandit feedback.

The transmission outcome can likewise be obtained for every portion, because the
portions are nested. This requires one departure from the per-timeslot behavior
described above, and it is made only while a trace is being collected. Instead
of sending the portion selected by a policy, the server sends the union of all
portions, ordering the tiles so that those of the smallest portion come first,
then those the next portion adds to it, and so on. The client reports how many
leading tiles arrived intact, in order, and in time, and a portion meets its
deadline if and only if that count covers all of its tiles. Since the leading
packets of a burst are unaffected by the packets that follow them, each prefix
is delivered as it would be if that portion were sent alone, up to any effect
of the longer burst on the timing of subsequent frames.

A recording made in this way contains $\{X_i(t)\}_{i \in [N]}$ and
$\{Y_i(t)\}_{i \in [N]}$ at every timeslot, so that all algorithms can afterwards
be evaluated offline on the identical data. In the evaluation, the optimal
portion $i^*$ in \eqref{system_model_regret} and \eqref{eq:relative_degradation}
is replaced by the empirical best fixed portion $\hat{i}$, the portion with the
largest realized cumulative reward $\sum_{t} X_i(t) Y_i(t)$ over all recorded
episodes, and both quantities are computed against its realized rewards
$Z_{\hat{i}}(t)$ within each episode; they measure performance on the recorded
outcomes rather than the asymptotic constant of Theorem~\ref{theorem:upper_bound}. 1/B-EXP3 uses the learning rate
$\eta = \sqrt{\log N / (T N)}$ of \cite[Algorithm~11.1]{lattimore2020bandit} with
$T = 3\times 10^3$.

We record $90$ consecutive passes over the head-motion trace and treat each
pass as one episode of horizon $T$. Every algorithm is
replayed on the identical episodes, so no difference between them can be
attributed to meeting a different channel. Each episode is replayed under
several random seeds and averaged before any interval is formed, so that the
reported variability is not the algorithms' own sampling. All reported
confidence intervals are two-sided $95\%$ $t$ intervals across the $90$
seed-averaged episode values. Because the episodes are recorded consecutively,
we also checked for episode-to-episode dependence: the lag-one autocorrelation
of the per-episode differences between AdaPort and each baseline is at most
$0.10$ in absolute value, and intervals computed with a dependence-robust
long-run variance estimate or a block bootstrap over episodes lead to the same
conclusions. If the
algorithms were instead run sequentially over a live link, each would experience
a different channel realization, and the observed performance gaps would be
confounded with channel variation.

\subsection{Performance}

\subsubsection{Regret}

Table~\ref{tab:regret} reports the cumulative regret over one episode of
$T=3\times10^3$ timeslots, averaged over the $90$ episodes. Without a window,
AdaPort incurs $37.2$ against $40.4$ for 1/B-TS, $49.9$ for 2/B/B-TS, $49.5$ for
1/B-EXP3, and $88.3$ for the heuristic.

The window length matters, and matters in both directions. A window of $50$
timeslots roughly halves the regret of AdaPort, while a window of $600$ leaves it
worse than using no window at all, so the parameter has to be chosen and choosing
it badly costs more than the window was worth.

\begin{table}[ht]
\centering
\caption{Cumulative regret per episode over $90$ episodes of $T=3\times10^3$ timeslots, $20$ seeds each, with a $95\%$ confidence interval across episodes.}
\label{tab:regret}

\begin{tabular}{lcc}
\toprule
algorithm & regret & $95\%$ CI \\
\midrule
AdaPort (2/F/B)             & 37.2 & [34.8, 39.6] \\
SW-AdaPort ($\tau=50$)      & \textbf{20.4} & [18.9, 22.0] \\
SW-AdaPort ($\tau=600$)     & 44.5 & [42.5, 46.5] \\
1/B-TS                      & 40.4 & [38.4, 42.4] \\
2/B/B-TS                    & 49.9 & [46.6, 53.3] \\
1/B-EXP3                    & 49.5 & [48.0, 51.1] \\
Heuristic                   & 88.3 & [85.3, 91.2] \\
\bottomrule
\end{tabular}
\end{table}

\subsubsection{Relative throughput degradation}

\begin{figure}[ht]
    \centering
    \includegraphics[width=\linewidth]{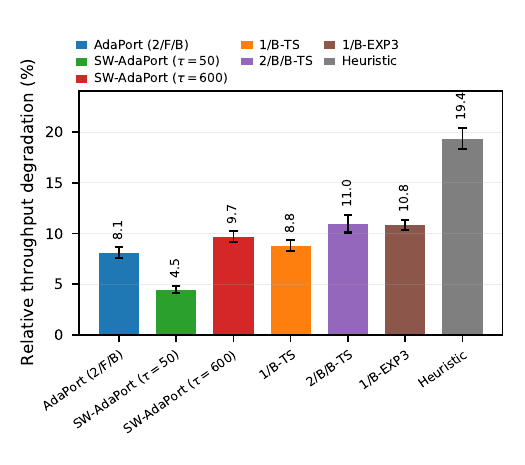}
    \caption{Relative throughput degradation compared to the empirical best fixed portion, over $90$ episodes of $T=3\times10^3$ timeslots, with a $95\%$ confidence interval across episodes. The bars ticked $\tau=50$ and $\tau=600$ are SW-AdaPort at those two window lengths. None of the four baselines uses a window.}
    \label{fig:degr}
\end{figure}

\paragraph{Comparison with 1/B-TS and 2/B/B-TS}

Figure~\ref{fig:degr} reports the relative throughput degradation
\eqref{eq:relative_degradation} for every algorithm considered, with a $95\%$
confidence interval taken across episodes. We read it in two passes. This
paragraph compares AdaPort with the learning baselines it is positioned relative
to and with the fixed heuristic, none of which uses a window.

AdaPort attains a lower degradation than 2/B/B-TS, $8.1\%$ against $11.0\%$, and
the two intervals do not overlap.
This is the comparison the 2/F/B model is principally about, because the two
algorithms differ only in whether the prediction outcome of an unplayed portion
is used. Using it is worth a reduction of $2.9$ percentage points in failed
deliveries on a real link.

The difference from 1/B-TS is in the same direction but smaller, $8.1\%$
against $8.8\%$. Since every algorithm is replayed on the identical episodes, the
comparison is paired: the per-episode difference between AdaPort and 1/B-TS has
mean $-0.7$ percentage points with a $95\%$ interval of $[-1.35, -0.10]$, which
excludes zero, and AdaPort is better on $55$ of the $90$ episodes. The margin is
small for a reason visible in the measured rates. The four prediction rates lie
between $0.889$ and $0.982$, so the full-information signal that AdaPort exploits
carries little information across portions here, and the advantage it confers is
correspondingly small. A trace on which the prediction rates were more widely
separated would be the natural setting in which to widen the margin.

All three learning algorithms stay far below the fixed heuristic, which attains
$19.4\%$ because it commits to one portion at the start of the session and never
revises that choice.

\paragraph{Comparison with 1/B-TS, 1/B-EXP3, and SW-AdaPort}
Recall that our i.i.d. assumptions about the prediction and transmission outcomes in Section~\ref{section:system_model} may not be realistic in practice. Therefore, we compare AdaPort, as well as 1/B-TS against 1/B-EXP3, which is a well-established algorithm designed for adversarial multi-armed bandit scenarios, and against SW-AdaPort, which uses the same selection rule as AdaPort but discards observations older than the window. These two baselines examine the same assumption in different ways. 1/B-EXP3 makes no stochastic assumption at all, whereas SW-AdaPort retains the stochastic model but restricts it to recent observations.

These two baselines are the second reading of Figure~\ref{fig:degr}. They put the
stationarity assumption under pressure from opposite sides, 1/B-EXP3 as its own
bar and SW-AdaPort as the two bars ticked by their window length.

AdaPort outperforms 1/B-EXP3, $8.1\%$ against $10.8\%$, and the two intervals do
not overlap. On these traces, the temporal
dependence reported next is therefore mild enough that an algorithm which
exploits stationarity still outperforms one that assumes nothing about it.

SW-AdaPort attains $4.5\%$ at $\tau = 50$ and $9.7\%$ at $\tau = 600$. Restricting
the posterior to recent observations therefore does help, so the dependence is
not merely mild enough to tolerate, but the gain is available only at a window
length that has to be chosen, and a window that is too long is worse than using
no window at all. We return to that choice in
Section~\ref{section:evaluation_window}.

\subsubsection{Temporal dependence of the feedback signals}

The testbed data allow the i.i.d.\ assumption of Section~\ref{section:system_model}
to be tested directly, and Figure~\ref{fig:autocorr} shows the result. The lag-one autocorrelation of the
transmission outcome is between $0.35$ and $0.39$ and is still between $0.11$
and $0.23$ at lag fifty, well outside the i.i.d.\ band of $\pm 0.036$.
The frame size varies slowly across the video, and the deadline turns a large
frame into a failed delivery, so runs of failures follow the content rather than
the channel. The prediction outcome has a comparable lag-one
autocorrelation but returns close to zero by lag two, which is consistent with
the short-term inertia of head motion.

The independence of the two signals from each other behaves differently, and it
is worth separating the two mechanisms that could couple them. The first runs
through the content. A head movement that causes a prediction error also moves
the viewport onto a different part of the frame, where the payload differs and a
larger frame is more likely to miss the deadline. This testbed exercises that
mechanism in full, and Table~\ref{tab:crossdep} measures it by conditioning the
transmission outcome of a portion on the prediction outcome of the same portion.
The $\phi$ coefficient, which is the correlation between two binary variables,
is at most $0.020$ in absolute value. For comparison, each signal correlates with
its own past at $0.19$ to $0.48$ at lag one, so the dependence between the
signals along this path is an order of magnitude weaker than the dependence
within either of them.

The second mechanism runs through the radio, where a head movement changes the
antenna orientation and the blockage between the client and the access point.
The client is stationary on a desk throughout a session and the viewport is
prescribed by the trace, so this mechanism is absent from the measurement by
construction. The result above should accordingly be read as evidence of weak same-arm
association under the content-mediated coupling, rather than as a verification
of the independence of $\mathbf{Y}(t)$ from the entire vector $\mathbf{X}(t)$
assumed in Section~\ref{section:system_model}.

\begin{table}[ht]
\centering
\caption{Transmission outcome conditioned on the prediction outcome of the same portion, over $90$ episodes.}
\label{tab:crossdep}

\begin{tabular}{cccc}
\toprule
portion $i$ & $\Pr(Y_i{=}1 \mid X_i{=}1)$ & $\Pr(Y_i{=}1 \mid X_i{=}0)$ & $\phi$ \\
\midrule
1 & 0.9177 & 0.9347 & $-0.0196$ \\
2 & 0.9096 & 0.9069 & $+0.0024$ \\
3 & 0.8690 & 0.8711 & $-0.0010$ \\
4 & 0.7695 & 0.7407 & $+0.0092$ \\
\bottomrule
\end{tabular}
\end{table}

\begin{figure}[ht]
    \centering
    \includegraphics[width=\linewidth]{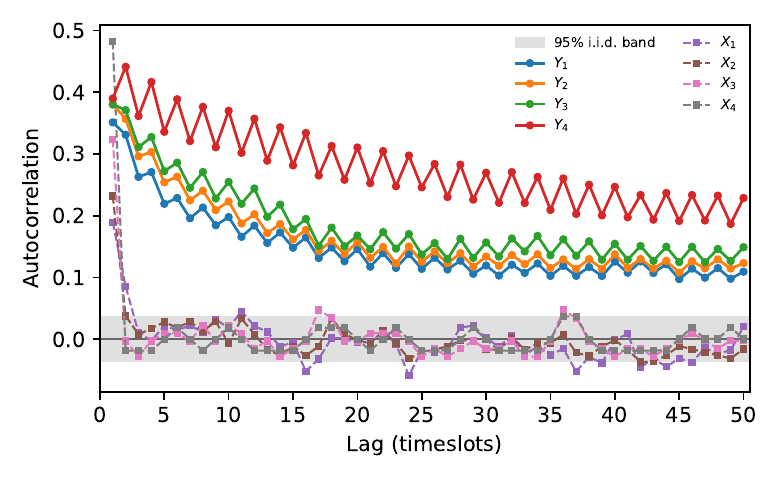}
    \caption{Autocorrelation of the transmission outcome $Y_i$, drawn solid, and of the prediction outcome $X_i$, drawn dashed, computed within each episode and averaged over the $90$ episodes. The shaded region is the pointwise $95\%$ reference band $\pm 1.96/\sqrt{T}$ for a single i.i.d.\ episode of length $T$.}
    \label{fig:autocorr}
\end{figure}

\subsubsection{The sliding-window variant and its window length}
\label{section:evaluation_window}

SW-AdaPort attains the lowest degradation of any algorithm considered when its
window is well chosen, $4.5\%$ at $\tau = 50$ against $8.1\%$ for AdaPort.

The benefit, however, is bought with a parameter that AdaPort does not have.
Figure~\ref{fig:window} shows the regret per episode against the window length.
The curve is U-shaped and crosses the unwindowed algorithm twice, so the window
helps only inside a bounded range and hurts outside it. Outside that range the
variant is worse than AdaPort,
$58.0$ at $\tau = 5$ and $44.5$ at $\tau = 600$ against $37.2$. Both failure modes
have a clear cause. A window that is too long
discards early observations without being short enough to follow the channel,
so it loses on both counts. A window that is too short leaves too few
observations per portion for the posteriors to be informative, since $\tau$
timeslots divide across $N$ portions, and the selection becomes more variable.

The rise at the right-hand end does not continue indefinitely. The window is
measured in timeslots and an episode lasts $T = 3\times10^3$ of them, so a window
of $\tau \geq T$ discards nothing within an episode and SW-AdaPort reduces
exactly to AdaPort. The curve therefore turns back towards the reference line
beyond the range plotted and meets it at $\tau = T$. We stop the sweep at
$\tau = 1200$ because the behaviour that matters, the minimum and the two
crossings, lies well below that.

A plausible explanation is that the window that works is matched to the
correlation length of the transmission process, which Figure~\ref{fig:autocorr}
places beyond fifty timeslots; we have not tested this directly. That length is a property of the deployment and is not known before
the system is run, and it will change with the environment. We therefore regard
the windowed variant as a useful diagnostic of how much the stationarity
assumption costs, and as a practical option where the correlation length can be
estimated online, rather than as a replacement for AdaPort.

\begin{figure}[ht]
    \centering
    \includegraphics[width=\linewidth]{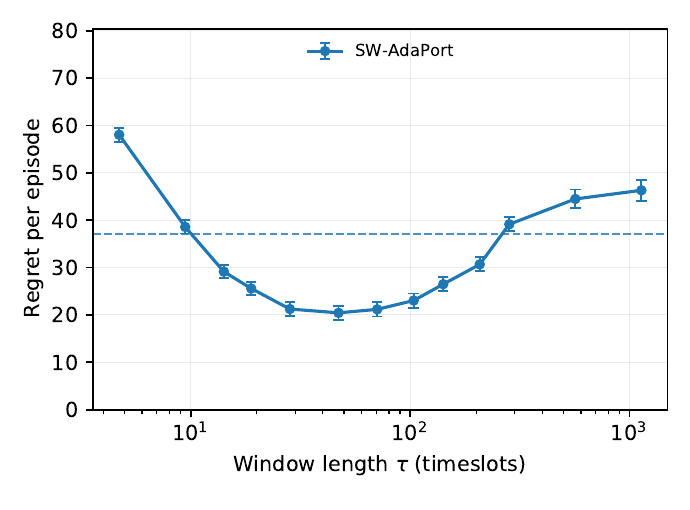}
    \caption{Regret per episode of SW-AdaPort against the window length $\tau$, with AdaPort shown as a dashed horizontal reference. Error bars are $95\%$ intervals across episodes. The variant is worse than AdaPort at both ends of the range.}
    \label{fig:window}
\end{figure}

\section{Conclusion and future work}
\label{section:conclusion}

In this paper, we investigated a hybrid feedback-driven online learning framework tailored to the interactive delivery of panoramic scenes over wireless networks. A critical observation is that, upon receiving the user's current head movement data, the system can determine whether each candidate portion successfully covers the user's viewport. This implies that the prediction feedback can be categorized as full-information feedback.
We first establish a theoretical improvement in the regret lower bound for 2/F/B, as compared to 2/B/B and 1/B feedback. This highlights the potential benefit of leveraging full-information predictions to enhance regret performance.
Building on this, we propose AdaPort, a hybrid learning algorithm that leverages both the
full-information feedback and bandit feedback and demonstrate that the proposed algorithm is asymptotically optimal.
Measurements on an end-to-end testbed show that AdaPort outperforms the state-of-the-art learning-based baselines and the heuristic minimum scene delivery scheme. In future work, we plan to investigate more realistic reward models that better capture the nuances of user experience. For example, we aim to refine the definition of delivery success by accounting for cases in which only a small portion of a segment fails to be delivered. Such partial failures may have minimal impact on the user’s viewing experience and thus should not necessarily be treated as complete delivery failures.  As noted in Section~\ref{section:system_model}, declaring a transmission successful whenever a prescribed fraction of its packets is delivered keeps $Y_i(t)$ binary, so AdaPort and its regret analysis continue to hold with $\beta_i$ redefined accordingly. A non-binary reward quantifying the degree of coverage is a more substantial extension, as it departs from the Bernoulli model. We also intend to explore alternative problem formulations, such as nonstationary bandits, to more accurately reflect the dynamics of real-world panoramic scene delivery.

\section*{Use of Generative AI Disclosure}

We acknowledge the use of ChatGPT 5.6~Pro and GPT-6 Astra in producing this revision of the original manuscript. In particular, they were used to find and fix gaps in the mathematical arguments in the paper. The authors also used Anthropic's Claude (Opus~5) to assist in debugging the end-to-end testbed described in
Section~\ref{section:empirical_evaluation}. All content was thoroughly reviewed by the authors.

%
%

\bibliographystyle{IEEEtran}
\bibliography{refs}

\appendices

\section{Proof of Theorem~\ref{theorem:lower_bound}}
\label{appendix:lower_bound}

To prove Theorem~\ref{theorem:lower_bound}, we require the following divergence decomposition lemma, which can be thought of as a version of Lemma 15.1 in \cite{lattimore2020bandit} for 2/F/B feedback. 

\begin{lemma}
\label{lemma:divergence_decomposition}
Consider an alternative problem instance with prediction distribution $p'_{\mathbf{X}}$ and transmission rates $\boldsymbol{\beta}'$, and the corresponding probability measure $\Pr'$ under this instance. Let $(\pi)$ be an arbitrary randomized policy that learns from 2/F/B feedback and let $p_\pi$ and $p_\pi'$ denote the probability mass functions of the sequence of observations $(i(t), \mathbf{X}(t), Y_{i(t)}(t))_{t=1}^T$  when the policy $(\pi)$ is used under the probability measures $\Pr$ and $\Pr'$ respectively. Then
\begin{equation}
\begin{aligned}
\KL{p_\pi}{p'_\pi} 
&= T \KL{p_\mathbf{X}}{p'_\mathbf{X}} \\
&\phantom{\,=\,}+ \sum_{i=1}^N \KL{p_{Y_i}}{p'_{Y_i}} \sum_{t=1}^T \Pr(i(t)=i)
\end{aligned}
\end{equation}
where $p_\mathbf{X}(\mathbf{x}) \triangleq \Pr(\mathbf{X}(1) = \mathbf{x})$, $p'_\mathbf{X}(\mathbf{x}) \triangleq \probprime(\mathbf{X}(1) = \mathbf{x})$, $p_{Y_i}(y_i) \triangleq \Pr(Y_i(1) = y_i)$, and $p'_{Y_i}(y_i) \triangleq \probprime(Y_i(1) = y_i)$.  
\end{lemma}

\begin{proof}
For convenience, let $\mathbf{W}_t \triangleq (i(t), \mathbf{X}(t), Y_{i(t)}(t))$ denote the vector of observations from timeslot $t$. Fix an arbitrary randomized policy $(\pi)$ specified by a probability mass function $\pi_t( \,\cdot\, \mid \mathbf{w}_{1:t-1} ) : [N] \to [0,1]$ for each timeslot $t$ and possible trajectory of prior observations $\mathbf{w}_{1:t-1} \in \mathrm{support}(\mathbf{W}_{1:t-1})$. The policy therefore plays the arm $i(t) \sim \pi_t( \,\cdot\, \mid \mathbf{W}_{1:t-1} )$ in each timeslot $t$. 

By the chain rule of probability, we have 
\begin{equation}
\label{eq:chain_rule_1}
\Pr(\mathbf{W}_{1:T} = \mathbf{w}_{1:T}) = \prod_{t=1}^T \Pr(\mathbf{W}_t = \mathbf{w}_t \mid \mathbf{W}_{1:t-1} = \mathbf{w}_{1:t-1})
\end{equation}
(taking the event $\{ \mathbf{W}_{1:t-1} = \mathbf{w}_{1:t-1} \}$ when $t=1$ to be equal to the sample space for notational convenience). Similarly,
\begin{equation}
\label{eq:chain_rule_2}
\probprime(\mathbf{W}_{1:T} = \mathbf{w}_{1:T}) = \prod_{t=1}^T \probprime(\mathbf{W}_t = \mathbf{w}_t \mid \mathbf{W}_{1:t-1} = \mathbf{w}_{1:t-1}).
\end{equation}
For each timeslot $t$, define the function 
\begin{equation}
f_t(\mathbf{w}_{1:t}) \triangleq \log \frac{\Pr(\mathbf{W}_t = \mathbf{w}_t \mid \mathbf{W}_{1:t-1} = \mathbf{w}_{1:t-1})}{\probprime(\mathbf{W}_t = \mathbf{w}_t \mid \mathbf{W}_{1:t-1} = \mathbf{w}_{1:t-1})}.
\end{equation}
Then it follows from \eqref{eq:chain_rule_1} and \eqref{eq:chain_rule_2} that
\begin{equation}
\begin{aligned}
\log\frac{\Pr(\mathbf{W}_{1:T} = \mathbf{w}_{1:T})}{\probprime(\mathbf{W}_{1:T} = \mathbf{w}_{1:T})}
= \sum_{t=1}^T f_t(\mathbf{w}_{1:t}).
\end{aligned}
\end{equation}
Then we have
\begin{equation}
\label{eq:kl_initial_decomp}
\begin{aligned}
\KL{p_\pi}{p'_\pi}
&= \sum_{\mathbf{w}_{1:T}} \Pr(\mathbf{W}_{1:T} = \mathbf{w}_{1:T}) \\
&\qquad\qquad\qquad\qquad \cdot \log\frac{\Pr(\mathbf{W}_{1:T} = \mathbf{w}_{1:T})}{\probprime(\mathbf{W}_{1:T} = \mathbf{w}_{1:T})} \\
&= \sum_{\mathbf{w}_{1:T}} \Pr(\mathbf{W}_{1:T} = \mathbf{w}_{1:T}) \sum_{t=1}^T f_t(\mathbf{w}_{1:t}) \\
&\overset{(a)}{=} \bE\mleft[ \sum_{t=1}^T f_t(\mathbf{W}_{1:t}) \mright] = \sum_{t=1}^T \bE\mleft[  f_t(\mathbf{W}_{1:t}) \mright] \\
&= \sum_{t=1}^T \sum_{\mathbf{w}_{1:t}} \prob(\mathbf{W}_{1:t} = \mathbf{w}_{1:t}) \\
&\qquad\qquad\qquad\qquad \cdot f_t(\mathbf{w}_{1:t}) \\
\end{aligned}
\end{equation}
where $(a)$ is by the law of the unconscious statistician. Fix a timeslot $t$ and a sample path $\mathbf{w}_{1:t}$ such that $\mathbf{w}_t = (i,\mathbf{x},y_i)$. 
Then
\begin{equation}
\label{eq:f_t_fixed_decomp}
\begin{aligned}
&f_t(\mathbf{w}_{1:t}) 
= \log \frac{\Pr(\mathbf{W}_t = (i,\mathbf{x},y_i) \mid \mathbf{W}_{1:t-1} = \mathbf{w}_{1:t-1})}{\probprime(\mathbf{W}_t = (i,\mathbf{x},y_i) \mid \mathbf{W}_{1:t-1} = \mathbf{w}_{1:t-1})} \\
&= \log \frac{\Pr(i(t) = i, \mathbf{X}(t) = \mathbf{x}, Y_i(t) = y_i \mid \mathbf{W}_{1:t-1} = \mathbf{w}_{1:t-1})}{\probprime(i(t) = i, \mathbf{X}(t) = \mathbf{x}, Y_i(t) = y_i \mid \mathbf{W}_{1:t-1} = \mathbf{w}_{1:t-1})} \\
&\overset{(a)}{=} \log \frac{\Pr(i(t) = i \mid \mathbf{W}_{1:t-1} = \mathbf{w}_{1:t-1}) \,p_\mathbf{X}(\mathbf{x})\,p_{Y_i}(y_i)}{\probprime(i(t) = i \mid \mathbf{W}_{1:t-1} = \mathbf{w}_{1:t-1})\,p'_\mathbf{X}(\mathbf{x})\,p'_{Y_i}(y_i)} \\
&\overset{(b)}{=} \log \frac{\pi_t(i \mid \mathbf{w}_{1:t-1}) \,p_\mathbf{X}(\mathbf{x})\,p_{Y_i}(y_i)}{\pi_t(i \mid \mathbf{w}_{1:t-1})\,p'_\mathbf{X}(\mathbf{x})\,p'_{Y_i}(y_i)} = \log \frac{p_\mathbf{X}(\mathbf{x})\,p_{Y_i}(y_i)}{p'_\mathbf{X}(\mathbf{x})\,p'_{Y_i}(y_i)} \\
\end{aligned}
\end{equation}
where $(a)$ is because $\mathbf{X}(t)$ and $Y_{i}(t)$ are independent and both i.i.d. over time, and therefore are independent from the history $\mathbf{W}_{1:t-1}$ and decision $i(t)$, and $(b)$ is because only the parameters $(\boldsymbol{\alpha},\boldsymbol{\beta})$ change to $(\boldsymbol{\alpha}',\boldsymbol{\beta}')$ between the probability measures $\Pr$ and $\probprime$ and therefore the additional randomness introduced by the randomized policy is distributed the same under both probability measures. It follows from \eqref{eq:kl_initial_decomp} that 
\begin{equation}
\label{eq:kl_second_decomp}
\begin{aligned}
&\KL{p_\pi}{p'_\pi} = \sum_{t=1}^T \sum_{\mathbf{w}_{1:t}} \prob(\mathbf{W}_{1:t} = \mathbf{w}_{1:t}) f_t(\mathbf{w}_{1:t}) \\
&= \sum_{t=1}^T \sum_{(i,\mathbf{x},y_i)}\,\sum_{\mathbf{w}_{1:t} \,:\, \mathbf{w}_t = (i,\mathbf{x},y_i)} \prob(\mathbf{W}_{1:t} = \mathbf{w}_{1:t}) f_t(\mathbf{w}_{1:t}) \\
&\overset{(a)}{=} \sum_{t=1}^T \sum_{(i,\mathbf{x},y_i)}\,\sum_{\mathbf{w}_{1:t} \,:\, \mathbf{w}_t = (i,\mathbf{x},y_i)} \prob(\mathbf{W}_{1:t} = \mathbf{w}_{1:t}) \\
&\qquad\qquad\qquad\qquad \cdot \log \frac{p_\mathbf{X}(\mathbf{x})\,p_{Y_i}(y_i)}{p'_\mathbf{X}(\mathbf{x})\,p'_{Y_i}(y_i)} \\
&= \sum_{t=1}^T \sum_{(i,\mathbf{x},y_i)}\log \frac{p_\mathbf{X}(\mathbf{x})\,p_{Y_i}(y_i)}{p'_\mathbf{X}(\mathbf{x})\,p'_{Y_i}(y_i)} \\
&\qquad\qquad\qquad\qquad \cdot \sum_{\mathbf{w}_{1:t} \,:\, \mathbf{w}_t = (i,\mathbf{x},y_i)} \prob(\mathbf{W}_{1:t} = \mathbf{w}_{1:t})  \\
&\overset{(b)}{=} \sum_{t=1}^T \sum_{(i,\mathbf{x},y_i)}\log \frac{p_\mathbf{X}(\mathbf{x})\,p_{Y_i}(y_i)}{p'_\mathbf{X}(\mathbf{x})\,p'_{Y_i}(y_i)} \\
&\qquad\qquad\qquad\qquad \cdot \prob(\mathbf{W}_t = (i,\mathbf{x},y_i))  \\
\end{aligned}
\end{equation}
where $(a)$ is from \eqref{eq:f_t_fixed_decomp} and $(b)$ applies the law of total probability. Since $i(t)$, $\mathbf{X}(t)$, and $Y_i(t)$ are independent, for each timeslot $t$ and observation vector $(i,\mathbf{x},y_i)$, we have
\begin{equation}
\begin{aligned}
\prob(\mathbf{W}_t = (i,\mathbf{x},y_i)) 
&= \prob(i(t) = i, \mathbf{X}(t) = \mathbf{x}, Y_i(t) = y_i) \\
&= \prob(i(t) = i) \,p_\mathbf{X}(\mathbf{x}) \,p_{Y_i}(y_i). \\
\end{aligned} 
\end{equation}
Plugging the above into \eqref{eq:kl_second_decomp} gives
\begin{equation}
\begin{aligned}
&\KL{p_\pi}{p'_\pi} \\
&= \sum_{t=1}^T \sum_{(i,\mathbf{x},y_i)}\log \frac{p_\mathbf{X}(\mathbf{x})\,p_{Y_i}(y_i)}{p'_\mathbf{X}(\mathbf{x})\,p'_{Y_i}(y_i)} \prob(i(t) = i) \,p_\mathbf{X}(\mathbf{x}) \,p_{Y_i}(y_i) \\
&= \sum_{t=1}^T \underbrace{\sum_{i=1}^N \prob(i(t) = i)}_{=\,1} \underbrace{\sum_{\mathbf{x}} p_\mathbf{X}(\mathbf{x}) \log \frac{p_\mathbf{X}(\mathbf{x})}{p'_\mathbf{X}(\mathbf{x})}}_{\KL{p_\mathbf{X}}{p'_\mathbf{X}}} \underbrace{\sum_{y_i} p_{Y_i}(y_i)}_{=\,1}  \\
&\phantom{\,=\,}+ \sum_{t=1}^T \sum_{i=1}^N \prob(i(t) = i) \underbrace{\sum_{\mathbf{x}} p_\mathbf{X}(\mathbf{x}) }_{\,=1} \underbrace{\sum_{y_i} p_{Y_i}(y_i) \log \frac{p_{Y_i}(y_i)}{p'_{Y_i}(y_i)}}_{\KL{p_{Y_i}}{p'_{Y_i}}} \\
&= T \KL{p_\mathbf{X}}{p'_\mathbf{X}} + \sum_{i=1}^N \KL{p_{Y_i}}{p'_{Y_i}} \sum_{t=1}^T \Pr(i(t)=i).
\end{aligned}
\end{equation}

\end{proof}

We now prove Theorem~\ref{theorem:lower_bound}.
Fix an arbitrary randomized policy $(\pi)$ that is consistent under 2/F/B feedback in the sense of Theorem~\ref{theorem:lower_bound}. Fix an arm $j \in \mathcal{I}$, i.e., $j \neq i^*$ with $\alpha_j > \alpha_{i^*}\beta_{i^*}$, and fix $\lambda \in {(\Delta_j, \alpha_j(1-\beta_j))}$. We construct a new problem instance in which the distribution of $\mathbf{X}(t)$ is unchanged (so that $\boldsymbol{\alpha}' = \boldsymbol{\alpha}$) and the transmission rates are $\boldsymbol{\beta}'$ with $\beta_j' = \beta_j + \lambda / \alpha_j$, and $\beta_i' = \beta_i$ for all $i \neq j$.
The interval for $\lambda$ is nonempty since $\Delta_j<\alpha_j\left(1-\beta_j\right)$ when $\alpha_j>\alpha_{i^*} \beta_{i^*}$, and every such $\lambda$ gives $\beta_j' \in (\beta_j, 1)$.

Observe that $\alpha_j^{\prime} \beta_j^{\prime}=\alpha_j \beta_j+\lambda>\alpha_{i^*} \beta_{i^*}$ and therefore $j \in \operatorname{argmax}_{i \in[N]} \alpha_i^{\prime} \beta_i^{\prime}$ is an optimal arm under the bandit setting with parameters $\boldsymbol{\alpha^{\prime}}, \boldsymbol{\beta^{\prime}}$ and the reward gaps $\Delta_i^{\prime} \triangleq$ $\alpha_j^{\prime} \beta_j^{\prime}-\alpha_i \beta_i$ satisfy
\begin{equation}
\Delta_i^{\prime}=\alpha_j \beta_j+\lambda-\alpha_i \beta_i=\lambda-\left(\alpha_i \beta_i-\alpha_j \beta_j\right) \geq \lambda-\Delta_j, \quad \forall i \neq j.    
\end{equation}
Let $\probprime$ denote the probability measure and $\bE^{\prime}$ denote the expectation operator in the setting with parameters $(\boldsymbol{\alpha}',\boldsymbol{\beta}')$, and let $R^{\prime}(T) \triangleq T \alpha_j^{\prime} \beta_j^{\prime}-\sum_{t=1}^T \bE^{\prime}\left[Z_{i(t)}(t)\right]$ denote the cumulative regret in this setting. Then
\begin{equation}
\begin{aligned}
\label{eqn:two_regrets}
R^{\prime}(T) & =\sum_{i \neq j} \Delta_i^{\prime} \bE^{\prime}\left[n_i(T)\right] \geq\left(\lambda-\Delta_j\right)\left(T-\bE^{\prime}\left[n_j(T)\right]\right), \\
R(T) & =\sum_{i \neq i^*} \Delta_i \bE\left[n_i(T)\right] \geq \Delta_j \bE\left[n_j(T)\right]
\end{aligned}    
\end{equation}
Define the event $\mathcal{E} \triangleq\left\{n_j(T)<T / 2\right\}$, we have 
\begin{equation}
\begin{aligned}
\label{eqn:regret1_lower}
& T-\bE^{\prime}\left[n_j(T)\right] \\
& =\bE^{\prime}\left[T-n_j(T)\middle | \mathcal{E} \right]\probprime(\mathcal{E})+\bE^{\prime}\left[T-n_j(T)\middle | \mathcal{E}^{\mathrm{c}} \right]\probprime(\mathcal{E}^{\mathrm{c}}) \\
& \geq \bE^{\prime}\left[T-n_j(T)\middle | \mathcal{E} \right]\probprime(\mathcal{E})\geq   \frac{1}{2} T \probprime(\mathcal{E}), \\
\end{aligned}    
\end{equation}
and 
\begin{equation}
\begin{aligned}
\label{eqn:regret2_lower}
& \bE\left[n_j(T)\right]  =\bE\left[n_j(T)\middle | \mathcal{E} \right]\prob(\mathcal{E})+\bE\left[n_j(T)\middle | \mathcal{E}^{\mathrm{c}} \right]\prob(\mathcal{E}^{\mathrm{c}}) \\
& \geq \bE\left[n_j(T)\middle | \mathcal{E}^{\mathrm{c}}\right]\prob(\mathcal{E}^{\mathrm{c}})\geq   \frac{1}{2} T \prob(\mathcal{E}^{\mathrm{c}}).
\end{aligned}    
\end{equation}
Inserting \eqref{eqn:regret1_lower} and \eqref{eqn:regret2_lower} into \eqref{eqn:two_regrets} gives
$R^{\prime}(T) \geq \frac{1}{2} \left(\lambda-\Delta_j\right)  T \probprime(\mathcal{E})$ and $R(T) \geq \frac{1}{2} \Delta_j  T \Pr\left(\mathcal{E}^{\mathrm{c}}\right)$. Therefore
\begin{equation}
\label{eq:sum_of_regrets}
 R^{\prime}(T)+R(T) \geq \frac{1}{2} T \min \left\{\left(\lambda-\Delta_j\right), \Delta_j\right\}\left[\Pr\left(\mathcal{E}^{\mathrm{c}}\right)+\probprime(\mathcal{E})\right].   
\end{equation}

Using the same notation as Lemma~\ref{lemma:divergence_decomposition}, let $p_\pi$ and $p_\pi'$ denote the probability mass functions of the sequence of observations $(i(t), \mathbf{X}(t), Y_{i(t)}(t))_{t=1}^T$  when the policy $(\pi)$ is used under the probability measures $\Pr$ and $\Pr'$ respectively. Note that $n_j(T)$ is a function of $(i(t), \mathbf{X}(t), Y_{i(t)}(t))_{t=1}^T$. Then we can write $\Pr(\mathcal{E}) = \sum_{\omega \in \mathcal{E}} p_\pi(\omega)$ (mutatis mutandis for $\probprime(\mathcal{E})$) and therefore
\begin{equation}
\begin{aligned}
\Pr(\mathcal{E})-\probprime(\mathcal{E}) 
&= \sum_{\omega \in \mathcal{E}} p_\pi(\omega) - \sum_{\omega \in \mathcal{E}} p'_\pi(\omega) \\
&\leq \sup_{\text{event }A} \left| \sum_{\omega \in A} p_\pi(\omega) - \sum_{\omega \in A} p'_\pi(\omega) \right| \\
&= D_{\mathrm{TV}}\left(p_\pi, p'_\pi\right)
\overset{(a)}{\leq} 1-\frac{1}{2} e^{-\KL{p_\pi}{p'_\pi}}
\end{aligned}
\end{equation}
where  $D_{\mathrm{TV}}$ denotes the total variation distance and $(a)$ applies the Bretagnolle-Huber inequality.
Rearranging and using the fact that $\Pr\left(\mathcal{E}^{\mathrm{c}}\right) = 1- \Pr\left(\mathcal{E}\right)$ gives
\begin{equation}
\label{eq:sum_of_probabilities}
\begin{aligned}
\Pr\left(\mathcal{E}^{\mathrm{c}}\right)+\probprime(\mathcal{E}) 
&\geq \frac{1}{2} e^{-\KL{p_\pi}{p'_\pi}} \\
&\stackrel{(a)}{=} \frac{1}{2} e^{-\KL{p_{Y_j}}{p'_{Y_j}} \bE\left[n_j(T)\right]},   
\end{aligned}
\end{equation}
where $(a)$ is from Lemma \ref{lemma:divergence_decomposition} (note that $\KL{p_\mathbf{X}}{p'_\mathbf{X}}=0$ since $p_\mathbf{X} = p'_\mathbf{X}$ by construction and each $\KL{p_{Y_i}}{p'_{Y_i}}=0$ since $\beta_i=\beta_i^{\prime}$ for $i \neq j$). Substituting \eqref{eq:sum_of_probabilities} into \eqref{eq:sum_of_regrets} and rearranging gives
\begin{equation}
e^{\KL{p_{Y_j}}{p'_{Y_j}} \bE\left[n_j(T)\right]} \geq \frac{\frac{1}{4} T \min \left\{\left(\lambda-\Delta_j\right), \Delta_j\right\}}{R^{\prime}(T)+R(T)} .    
\end{equation}
Taking the logarithm of both sides, dividing by $\log T$ and rearranging gives
\begin{equation}
\begin{aligned}
\frac{\bE\left[n_j(T)\right]}{\log T} & \geq \frac{1}{\KL{p_{Y_j}}{p'_{Y_j}}} +\frac{\log \left(\frac{1}{4} \min \left\{\left(\lambda-\Delta_j\right), \Delta_j\right\}\right)}{\KL{p_{Y_j}}{p'_{Y_j}} \log T} \\
& -\frac{\log \left(R^{\prime}(T)+R(T)\right)}{\KL{p_{Y_j}}{p'_{Y_j}} \log T} .    
\end{aligned}
\end{equation}

Taking the limit inferior as $T \rightarrow \infty$ of both sides, we need only inspect the last term on the right-hand side. By consistency on both instances, $R(T) = o(T^\delta)$ and $R'(T) = o(T^\delta)$ for every $\delta > 0$. Then $\forall \delta>0$, there exists a constant $C_\delta>0$ such that $R^{\prime}(T)+R(T) \leq C_\delta T^\delta$ and therefore
\begin{equation}
\begin{aligned}
\limsup _{T \rightarrow \infty} \frac{\log \left(R^{\prime}(T)+R(T)\right)}{\log T} &\leq \limsup _{T \rightarrow \infty} \frac{\log \left(C_\delta\right)+\delta \log T}{\log T} \\
&=\delta.
\end{aligned}
\end{equation}

Since $\delta$ can be taken arbitrarily small and positive, it follows that
\begin{equation}
\limsup _{T \rightarrow \infty} \frac{\log \left(R^{\prime}(T)+R(T)\right)}{\log T}\leq0.
\end{equation}
and so $\liminf_{T \rightarrow \infty} \frac{\bE\left[n_j(T)\right]}{\log T} \geq \frac{1}{\KL{p_{Y_j}}{p'_{Y_j}}}$, where $\KL{p_{Y_j}}{p'_{Y_j}} = d\mleft(\beta_j, \beta_j + \lambda/\alpha_j\mright)$.  Since this bound holds for every $\lambda \in\left(\Delta_j, \alpha_j\left(1-\beta_j\right)\right)$, letting $\lambda \downarrow \Delta_j$ gives $\beta_j' \downarrow \frac{\alpha_{i^*}\beta_{i^*}}{\alpha_j}$, and continuity of $d(\beta_j,\cdot)$ yields the result.

\section{Proof of Lemma~\ref{lemma:upper_bound}}
\label{appendix:upper_bound}

Before proceeding to the proof, we first introduce a key definition and a supporting lemma (cf. Lemma~\ref{lemma:first_term_lemma}) to facilitate the analysis.
\begin{definition}
Define $p_{i^*, t}$ as the probability that the Thompson sample of arm $i^*$ is greater than the threshold $x \triangleq \frac{\alpha_{i^*}\beta_{i^*}-\epsilon_3}{\alpha_{i^*}-\epsilon_2}$ of Lemma~\ref{lemma:upper_bound} given the history until $t-1$, i.e.,
$p_{i^*, t}\triangleq\Pr\big(\theta_{\beta,i^*}(t)>x \,|\,\mathcal{F}_{t-1} \big) $, where $\mathcal{F}_{t} \triangleq \{i(\tau), X_{1}(\tau),\ldots, X_{N}(\tau), Y_{i(\tau)}(\tau),\tau=1,\ldots, t\}$ represents the history up to time $t$. Since $\theta_{\beta,i^*}(t)$ has a Beta distribution given $\mathcal{F}_{t-1}$ and $x<1$, we have $p_{i^*,t}>0$ on every history.
\end{definition}
For brevity, we also use 
$\mathcal{F}_{t}$ to denote the $\sigma$-algebra generated by this history.

\begin{lemma}
\label{lemma:first_term_lemma}
For all $t, i \neq i^*$, we have 
\begin{equation}
\begin{aligned}
\label{eq:lemma:first_term_lemma}
 & \operatorname{Pr}\left(i(t)=i, E_i^\mu(t), E_i^\theta(t), G_{i^*}^\mu(t) \mid \mathcal{F}_{t-1}\right)   \\
 & \leq \frac{\left(1-p_{i^*, t}\right)}{p_{i^*, t}} \operatorname{Pr}\left(i(t)=i^*, E_i^\mu(t), E_i^\theta(t), G_{i^*}^\mu(t) \mid \mathcal{F}_{t-1}\right).
\end{aligned}
\end{equation}
\end{lemma}

\begin{proof}
Fix $t$ and a history $F_{t-1}$ of $\mathcal{F}_{t-1}$. The events $E_i^\mu(t)$ and $G_{i^*}^\mu(t)$ are $\mathcal{F}_{t-1}$-measurable. If either fails on $F_{t-1}$, both sides of \eqref{eq:lemma:first_term_lemma} are zero, since $p_{i^*,t}>0$. Otherwise, $\ol{\alpha}_{i^*}(t) \geq \alpha_{i^*}-\epsilon_2$ on $F_{t-1}$, so $\{\theta_{\beta,i^*}(t) > x\} \subseteq \{\ol{\alpha}_{i^*}(t)\theta_{\beta,i^*}(t) > \alpha_{i^*}\beta_{i^*}-\epsilon_3\}$ and hence
\begin{equation}
\label{eq:pprime}
p'_t \triangleq \Pr\big(\ol{\alpha}_{i^*}(t)\theta_{\beta,i^*}(t) > \alpha_{i^*}\beta_{i^*}-\epsilon_3 \mid \mathcal{F}_{t-1}\big) \geq p_{i^*,t}.
\end{equation}
Let $B \triangleq \{\ol{\alpha}_j(t)\theta_{\beta,j}(t) \leq \alpha_{i^*}\beta_{i^*}-\epsilon_3 \text{ for all } j \neq i^*\}$. If $i(t)=i$ and $E_i^\theta(t)$ holds, then by the selection rule \eqref{alg:policy} every $\ol{\alpha}_j(t)\theta_{\beta,j}(t)$ is at most $\ol{\alpha}_i(t)\theta_{\beta,i}(t) \leq \alpha_{i^*}\beta_{i^*}-\epsilon_3$, so $B$ holds and $\ol{\alpha}_{i^*}(t)\theta_{\beta,i^*}(t) \leq \alpha_{i^*}\beta_{i^*}-\epsilon_3$. Since the Thompson samples are conditionally independent across arms given $\mathcal{F}_{t-1}$,
\begin{equation}
\label{eq:lemma3_upper}
\Pr\big(i(t)=i, E_i^\theta(t) \mid \mathcal{F}_{t-1}\big) \leq (1-p'_t)\Pr(B \mid \mathcal{F}_{t-1}).
\end{equation}
Conversely, on $B \cap \{\ol{\alpha}_{i^*}(t)\theta_{\beta,i^*}(t) > \alpha_{i^*}\beta_{i^*}-\epsilon_3\}$, arm $i^*$ is the unique maximizer in \eqref{alg:policy} and $E_i^\theta(t)$ holds, so
\begin{equation}
\label{eq:lemma3_lower}
\Pr\big(i(t)=i^*, E_i^\theta(t) \mid \mathcal{F}_{t-1}\big) \geq p'_t \Pr(B \mid \mathcal{F}_{t-1}).
\end{equation}
Since $E_i^\mu(t)$ and $G_{i^*}^\mu(t)$ hold on $F_{t-1}$, they can be added to the events in \eqref{eq:lemma3_upper} and \eqref{eq:lemma3_lower} without changing either probability. Combining the two displays with $(1-p'_t)/p'_t \leq (1-p_{i^*,t})/p_{i^*,t}$, which follows from \eqref{eq:pprime}, proves \eqref{eq:lemma:first_term_lemma}.
\end{proof}

The decomposition \eqref{eqn:initial_decomposition} and the concentration arguments for its last three terms follow the Thompson Sampling analysis of \cite{agrawal2017near}.

We now proceed with the proof of Lemma~\ref{lemma:upper_bound}.
We first prove \eqref{decompose_first_term}. Applying the law of total expectation followed by Lemma~\ref{lemma:first_term_lemma} to the LHS of \eqref{decompose_first_term} gives
\allowdisplaybreaks
\begin{align}
\label{eq:first_term_upperbound_1}
& \sum_{t=1}^T \operatorname{Pr}\left(i(t)=i, E_i^\mu(t), E_i^\theta(t), G_{i^*}^\mu(t)\right) \nonumber\\
& =\sum_{t=1}^T \mathbb{E}\left[\operatorname{Pr}\left(i(t)=i, E_i^\mu(t), E_i^\theta(t), G_{i^*}^\mu(t)\mid \mathcal{F}_{t-1}\right)\right] \nonumber \\
& \leq \sum_{t=1}^T \mathbb{E}\left[\frac{\left(1-p_{i^*, t}\right)}{p_{i^*, t}} \operatorname{Pr}\left(
  \substack{
    i(t)=i^* \\
    E_i^\mu(t) \\
    E_i^\theta(t) \\
    G_{i^*}^\mu(t)
  }
  \,\middle|\,
  \mathcal{F}_{t-1}
\right)\right] \nonumber\\
& \stackrel{(a)}{=}\sum_{t=1}^T \mathbb{E}\left[\mathbb{E}\left[\left.\frac{\left(1-p_{i^*, t}\right)}{p_{i^*, t}} \id\left(
  \substack{
    i(t)=i^* \\
    E_i^\theta(t) \\
    E_i^\mu(t) \\
    G_{i^*}^\mu(t)
  }
\right)
 \right\rvert\, \mathcal{F}_{t-1}\right]\right] \nonumber \\
& \stackrel{(b)}{=}\sum_{t=1}^T \mathbb{E}\left[\frac{\left(1-p_{i^*, t}\right)}{p_{i^*, t}} \id\left(i(t)=i^*, E_i^\theta(t), E_i^\mu(t),  G_{i^*}^\mu(t)\right)\right] \nonumber\\
& = \sum_{t=1}^T \mathbb{E}\left[\left(\frac{1}{p_{i^*, t}}-1 \right) \id\left(i(t)=i^*, E_i^\theta(t), E_i^\mu(t),  G_{i^*}^\mu(t)\right)\right] \nonumber \\
&\leq  \bE\Bigg[\sum_{t=1}^T \left( \frac{1}{p_{i^*, t}} -1 \right) \id\mleft(i(t)=i^* \mright)\Bigg] ,
\end{align}
where step $(a)$ is because $p_{i^*, t}$ is $\mathcal{F}_{t-1}$-measurable; $(b)$ follows from the law of total expectation.

Since $\theta_{\beta,i}(t)$ is drawn from $\Beta(S_{\beta,i}(t)+1, F_{\beta,i}(t)+1)$ given $\mathcal{F}_{t-1}$, the following holds almost surely for any arm $i \in [N]$:
\begin{equation}
\begin{aligned}
\label{eq:first_term_upperbound_2}
& \Pr\big(\theta_{\beta,i}(t)> x  \mid \mathcal{F}_{t-1} \big) \\
& = \Pr\big(\theta_{\beta,i}(t)> x\mid \{i(\tau),Y_{i(\tau)}(\tau),\tau=1,\ldots, t-1 \} \big) \\
& = \Pr\big(\theta_{\beta,i}(t)> x\mid S_{\beta,i}(t),  F_{\beta,i}(t) \big).
\end{aligned}
\end{equation}

To bound the right-hand side of \eqref{eq:first_term_upperbound_1}, we use the stack-of-rewards construction \cite[Section~4.6]{lattimore2020bandit}. For each arm $i$, let $U_{i,1}, U_{i,2}, \ldots$ be i.i.d.\ $\text{Bernoulli}(\beta_i)$ random variables, independent of $\{\mathbf{X}(t)\}_{t \geq 1}$ and of the randomness used to draw the Thompson samples, and let the $s$-th play of arm $i$ receive the transmission outcome $Y_i(\tau_{s,i}) = U_{i,s}$. This construction leaves the joint distribution of $(i(t), \mathbf{X}(t), Y_{i(t)}(t))_{t \geq 1}$ unchanged. Let $S_{i,s} \triangleq \sum_{u=1}^{s} U_{i,u}$, so that for every deterministic $s$, $S_{i,s}$ has the binomial distribution with $s$ trials and success probability $\beta_i$, and along the trajectory $S_{\beta,i}(t) = S_{i,n_i(t-1)}$. We also recall the relationship between the Beta and binomial CDFs,
\begin{equation}
\label{eq:beta_binomial}
F^{\text{Beta}}_{a,b}(y) = 1 - F^B_{a+b-1,y}(a-1) \quad \text{for all positive integers } a, b.
\end{equation}
For $s \geq 0$, define
\begin{equation}
\begin{aligned}
\label{eq:q_s_g_s}
q_s &\triangleq 1 - F^{\text{Beta}}_{S_{i^*,s}+1,\, s - S_{i^*,s}+1}(x) = F^{B}_{s+1,x}(S_{i^*,s}), \\
g_s &\triangleq \frac{1}{q_s} - 1,
\end{aligned}
\end{equation}
where the equality uses \eqref{eq:beta_binomial}. Since, given $\mathcal{F}_{t-1}$, $\theta_{\beta,i^*}(t) \sim \Beta(S_{\beta,i^*}(t)+1, F_{\beta,i^*}(t)+1)$ with $S_{\beta,i^*}(t) + F_{\beta,i^*}(t) = n_{i^*}(t-1)$, we have $p_{i^*,t} = q_{n_{i^*}(t-1)}$. Since $n_{i^*}(t-1)$ takes each value $s \in \{0, \ldots, n_{i^*}(T)-1\}$ exactly once over the timeslots $t \leq T$ with $i(t) = i^*$, we have, on every sample path,
\begin{equation}
\begin{aligned}
\label{eq:first_term_upperbound_3}
&\sum_{t=1}^T \left( \frac{1}{p_{i^*,t}}-1 \right)\id\mleft(i(t)=i^* \mright)
= \sum_{s=0}^{n_{i^*}(T)-1} g_s \leq \sum_{s=0}^{\infty} g_s.
\end{aligned}
\end{equation}
Taking expectations and applying Tonelli's theorem, the right-hand side of \eqref{eq:first_term_upperbound_1} is at most $\sum_{s=0}^{\infty} \bE[g_s]$. The following lemma shows that this series is finite.

\begin{lemma}
\label{lem:posterior_sum}
Let $0 < x < p \leq 1$, let $S_s$ be the sum of $s$ i.i.d.\ $\text{Bernoulli}(p)$ random variables, and let $q_s \triangleq F^B_{s+1,x}(S_s)$ and $g_s \triangleq q_s^{-1} - 1$. Then $q_s > 0$ for every $s$ and every realization of $S_s$, and
\begin{equation}
\label{eq:C_post}
C_{\mathrm{post}}(p,x) \triangleq \sum_{s=0}^{\infty} \bE[g_s] < \infty.
\end{equation}
\end{lemma}
\begin{proof}
Let $f^B_{s,a}(k) \triangleq \binom{s}{k} a^k (1-a)^{s-k}$ denote the probability mass function of the binomial distribution with $s$ trials and success probability $a$. Since $q_s \geq F^B_{s+1,x}(0) = (1-x)^{s+1} > 0$, each $\bE[g_s]$ is finite, so it suffices to show that $\bE[g_s]$ is summable over all sufficiently large $s$.

First suppose $p < 1$. For $u \in [0,1]$, define the increasing affine function $\ell(u) \triangleq u \log\frac{p}{x} + (1-u) \log\frac{1-p}{1-x}$. Since $\ell(x) = -d(x,p) < 0$, there exists $r \in (x,p)$ with $\ell(r) < 0$. Let $c \triangleq -\ell(r) > 0$ and choose $r_0 \in (x, r)$. We split $\bE[g_s]$ according to whether $S_s \leq sr$.

For every integer $k \in \{0, \ldots, s\}$, the event of exactly $k$ successes in the first $s$ trials and a failure in the last trial is contained in the event of at most $k$ successes in $s+1$ trials, so $F^B_{s+1,x}(k) \geq (1-x) f^B_{s,x}(k)$. For $s \geq 1$ and $k \leq sr$, $f^B_{s,p}(k)/f^B_{s,x}(k) = e^{s\ell(k/s)} \leq e^{-cs}$. Consequently,
\begin{equation}
\label{eq:odds_low}
\bE\left[g_s \id(S_s \leq sr)\right] \leq \sum_{k=0}^{\lfloor sr \rfloor} \frac{f^B_{s,p}(k)}{F^B_{s+1,x}(k)} \leq \frac{s+1}{1-x} e^{-cs}.
\end{equation}

For all sufficiently large $s$, $sr \geq (s+1) r_0$. On the event $S_s > sr$, monotonicity of the CDF and the Chernoff--Hoeffding bound give $1 - q_s = 1 - F^B_{s+1,x}(S_s) \leq 1 - F^B_{s+1,x}((s+1)r_0) \leq e^{-(s+1) d(r_0, x)}$, which is at most $1/2$ for all sufficiently large $s$. Since $g_s = (1-q_s)/q_s \leq 2(1-q_s)$ when $q_s \geq 1/2$,
\begin{equation}
\label{eq:odds_high}
\bE\left[g_s \id(S_s > sr)\right] \leq 2 e^{-(s+1) d(r_0,x)}
\end{equation}
for all sufficiently large $s$. The right-hand sides of \eqref{eq:odds_low} and \eqref{eq:odds_high} are summable over $s$, which proves \eqref{eq:C_post} when $p<1$.

If $p = 1$, then $S_s = s$ almost surely, $q_s = F^B_{s+1,x}(s) = 1 - x^{s+1}$, and $g_s = \frac{x^{s+1}}{1-x^{s+1}} \leq \frac{x^{s+1}}{1-x}$. Therefore $C_{\mathrm{post}}(1,x) \leq \frac{x}{(1-x)^2} < \infty$.
\end{proof}

Applying Lemma~\ref{lem:posterior_sum} with $p = \beta_{i^*}$, $S_s = S_{i^*,s}$, and the threshold $x$, which satisfies $0 < x < \beta_{i^*}$ by the choice of $\epsilon_3$ and the second inequality in \eqref{eqn:epsilon_range}, and inserting into \eqref{eq:first_term_upperbound_1} gives
\begin{equation}
\label{eq:upper_bound_first_term}
\sum_{t=1}^T \Pr\left(i(t)=i, E_i^\mu(t), E_i^\theta(t), G_{i^*}^\mu(t)\right) \leq C_{\mathrm{post}}(\beta_{i^*}, x) = M_1.
\end{equation}

Next, we prove \eqref{decompose_second_term}. 
\paragraph{Case 1: $\alpha_i \geq \alpha_{i^*}\beta_{i^*}$}
For arm $i\neq i^*$ satisfying $\alpha_i\geq\alpha_{i^*}\beta_{i^*}$, we can have
\begin{align}
\label{prop:second_term}
& \sum_{t=1}^T \Pr\left(i(t)=i, \overline{E_i^\theta(t)}, E_i^\mu(t)\right) \nonumber \\
& = \sum_{t=1}^T \Pr\left(i(t)=i, n_i(t-1) \leq L_i(T), \overline{E_i^\theta(t)}, E_i^\mu(t)\right) \nonumber \\
& +\sum_{t=1}^T \operatorname{Pr}\left(i(t)=i, n_i(t-1)>L_i(T), \ol{E_i^\theta(t)}, E_i^\mu(t)\right) , 
\end{align}
where $L_i(T)=\frac{\log T}{d(\beta_i+\epsilon_1,\frac{\alpha_{i^*}\beta_{i^*}-\epsilon_3}{\alpha_i+\epsilon_1})}$. 

The first term in \eqref{prop:second_term} can be upper bounded as 
\begin{align}
\label{eq:upper_bound_second_term_not_enough_time}
& \sum_{t=1}^T \Pr\left(i(t)=i, n_i(t-1) \leq L_i(T), \overline{E_i^\theta(t)}, E_i^\mu(t)\right) \nonumber \\
& =\sum_{t=1}^T \bE\left[\id\left(i(t)=i, n_i(t-1) \leq L_i(T), \overline{E_i^\theta(t)}, E_i^\mu(t)\right) \right] \nonumber \\  
& = \bE\left[\sum_{s=1}^{n_i(T)}\id\left( n_i(\tau_{s,i}-1) \leq L_i(T), \overline{E_i^\theta(\tau_{s,i})}, E_i^\mu(\tau_{s,i})\right) \right] \nonumber \\  
& = \bE\left[\sum_{s=0}^{n_i(T)-1}\id\left( s \leq L_i(T), \overline{E_i^\theta(\tau_{s+1,i})}, E_i^\mu(\tau_{s+1,i})\right) \right] \nonumber \\  
& \leq L_i(T) +1.
\end{align}

Then we focus on bounding the second term in \eqref{prop:second_term}.
\begin{equation}
\resizebox{\linewidth}{!}{$
\begin{aligned}
\label{eq:upper_bound_second_term}
& \sum_{t=1}^T \operatorname{Pr}\left(i(t)=i, n_i(t-1)>L_i(T), \overline{E_i^\theta(t)}, E_i^\mu(t)\right) \\
& =\sum_{t=1}^T \mathbb{E}\left[\id\left(i(t)=i, n_i(t-1)>L_i(T), \overline{E_i^\theta(t)}, E_i^\mu(t)\right)\right] \\
& \stackrel{(a)}{=}\mathbb{E}\left[\sum_{t=1}^T \mathbb{E}\left[\id\left(i(t)=i, n_i(t-1)>L_i(T), \overline{E_i^\theta(t)}, E_i^\mu(t)\right) \mid \mathcal{F}_{t-1}\right]\right] \\
& \stackrel{(b)}{=}\mathbb{E}\left[\sum_{t=1}^T \id\left(n_i(t-1)>L_i(T), E_i^\mu(t) \right) \operatorname{Pr}\left(i(t)=i, \overline{E_i^\theta(t)} \mid \mathcal{F}_{t-1}\right)\right]  \\
& \leq \mathbb{E}\Bigg[\sum_{t=1}^T \id\left(n_i(t-1)>L_i(T), |\ol{\alpha}_i(t)-\alpha_i| \leq \epsilon_1, |\ol{\beta}_i(t)-\beta_i| \leq \epsilon_1 \right) \\ 
& \cdot\Pr\left(\ol{\alpha}_i(t)\theta_{\beta,i}(t) > \alpha_{i^*}\beta_{i^*}-\epsilon_3 \mid \mathcal{F}_{t-1}\right)\Bigg],
\end{aligned}
$}
\end{equation}
where step $(a)$ uses the law of total expectation and $(b)$ uses the fact that $\id(n_i(t-1)>L_i(T),E^{\mu}_i(t))$ is $\mathcal{F}_{t-1}$-measurable.

By definition, $S_{\beta,i}(t)= {\ol{\beta}_i(t)n_i(t-1)}$ and $F_{\beta,i}(t)= (1-\ol{\beta}_i(t))n_i(t-1)$, and therefore, $\theta_{\beta,i}(t)$ is a $\operatorname{Beta}\left(\ol{\beta}_i(t)n_i(t-1)+1,\left(1-\ol{\beta}_i(t)\right)n_i(t-1)+1\right)$ distributed random variable with integer parameters. Combing with \eqref{eq:first_term_upperbound_2} and \eqref{eq:beta_binomial} gives
\begin{equation}
\begin{aligned}
\label{eq:stochastic_dominance}
& \Pr\big(\theta_{\beta,i}(t)> y  \mid \mathcal{F}_{t-1} \big) \\
& = 1-F_{\ol{\beta}_i(t) n_i(t-1)+1,\left(1-\ol{\beta}_i(t) \right)n_i(t-1)+1}^{\text{Beta}}\left(y\right) \\
& = F_{ n_i(t-1)+1,y}^{B}\left(\ol{\beta}_i(t) n_i(t-1)\right).
\end{aligned}
\end{equation}

It suffices to bound the conditional probability in \eqref{eq:upper_bound_second_term} on histories $\mathcal{F}_{t-1} = F_{t-1}$ for which $n_i(t-1) > L_i(T)$ and $E_i^\mu(t)$ hold, since the indicator multiplying it vanishes on all other histories. Fix such a history; then
\begin{align}
\label{eqn:product_beta_edge1}
& \Pr\left(\ol{\alpha}_i(t)\theta_{\beta,i}(t) > \alpha_{i^*}\beta_{i^*}-\epsilon_3 \mid \mathcal{F}_{t-1}=F_{t-1}\right) \nonumber \\
& \leq \Pr\left(\theta_{\beta,i}(t)\geq \frac{\alpha_{i^*}\beta_{i^*}-\epsilon_3}{\alpha_i+\epsilon_1} \mid \mathcal{F}_{t-1}=F_{t-1} \right) \nonumber\\
& \stackrel{(a)}{\leq} F_{n_i(t-1)+1, \frac{\alpha_{i^*}\beta_{i^*}-\epsilon_3}{\alpha_i+\epsilon_1}}^{\mathrm{B}}((\beta_i+\epsilon_1)n_i(t-1)) \nonumber\\
& \stackrel{(b)}{\leq} F_{n_i(t-1)+1, \frac{\alpha_{i^*}\beta_{i^*}-\epsilon_3}{\alpha_i+\epsilon_1}}^{\mathrm{B}}((\beta_i+\epsilon_1)(n_i(t-1)+1)) \nonumber\\
& \stackrel{(c)}{\leq} e^{-\left(n_i(t-1)+1\right) d\left(\beta_i+\epsilon_1, \frac{\alpha_{i^*}\beta_{i^*}-\epsilon_3}{\alpha_i+\epsilon_1}\right)} 
 \stackrel{(d)}{\leq} \frac{1}{T},
\end{align}
where step $(a)$ uses \eqref{eq:stochastic_dominance} and $\overline{\beta}_i(t) n_i(t-1) \leq (\beta_i + \epsilon_1) n_i(t-1)$ under the sample path $\mathcal{F}_{t-1} = F_{t-1}$, together with the fact that the CDF of any random variable is non-decreasing; $(b)$ follows from the fact that CDF of any random variable is non-decreasing; $(c)$ uses the Chernoff-Hoffeding Bound; $(d)$ from $n_i(t-1)+1 > L_i(T)$.

Inserting \eqref{eqn:product_beta_edge1} into \eqref{eq:upper_bound_second_term} gives
\begin{align}
\label{eq:upper_bound_second_term_enogh_time}
\sum_{t=1}^T \operatorname{Pr}\left(i(t)=i, n_i(t-1)>L_i(T), \overline{E_i^\theta(t)}, E_i^\mu(t)\right) \leq 1 
\end{align}

After substituting \eqref{eq:upper_bound_second_term_not_enough_time} and \eqref{eq:upper_bound_second_term_enogh_time} into \eqref{prop:second_term}, for arm $i\neq i^*$ satisfying $\alpha_i\geq \alpha_{i^*}\beta_{i^*}$, we obtain
\begin{equation}
\label{eq:upper_bound_second_term_case1}
\sum_{t=1}^T \Pr\left(i(t)=i, \overline{E_i^\theta(t)}, E_i^\mu(t)\right) \leq \frac{\log T}{d(\beta_i+\epsilon_1,\frac{\alpha_{i^*}\beta_{i^*}-\epsilon_3}{\alpha_i+\epsilon_1})} +2.
\end{equation}

\paragraph{Case 2: $\alpha_i < \alpha_{i^*}\beta_{i^*}$}

For arm $i\neq i^*$ satisfying $\alpha_i<\alpha_{i^*}\beta_{i^*}$, we can choose any $\epsilon_4 \in (0, \alpha_{i^*}\beta_{i^*}-\alpha_{i}-\epsilon_3)$ such that
\begin{align}
\label{eq:alpha_is_small}
& \sum_{t=1}^T \Pr\left(i(t)=i, \overline{E_i^\theta(t)}, E_i^\mu(t)\right) \nonumber \\
& = \sum_{t=1}^T \Pr\left(i(t)=i, \ol{\alpha}_i(t)\theta_{\beta,i}(t)>\alpha_{i^*}\beta_{i^*}-\epsilon_3, E_i^\mu(t)\right) \nonumber \\
& \leq \sum_{t=1}^T \Pr\left( \ol{\alpha}_i(t)\geq \alpha_{i}+\epsilon_4 \right) \nonumber \\
& \stackrel{(a)} {\leq} 1+\sum_{t=1}^{\infty} \exp \left(-t d\left(\alpha_{i}+\epsilon_4, \alpha_{i}\right)\right)  \stackrel{(b)}{\leq} 1+\frac{1}{2\epsilon_4^2},
\end{align}
where step $(a)$ uses the Chernoff-Hoeffding bound; $(b)$ uses the Pinsker inequality.

Therefore, for arm $i\neq i^*$ satisfying $\alpha_i<\alpha_{i^*}\beta_{i^*}$, we obtain
\begin{equation}
\begin{aligned}
\label{eq:upper_bound_second_term_case2}
&\sum_{t=1}^T \Pr\left(i(t)=i, \overline{E_i^\theta(t)}, E_i^\mu(t)\right) \\
&\leq \inf_{\epsilon_4 \in (0,\alpha_{i^*}\beta_{i^*}-\alpha_i-\epsilon_3)} \frac{1}{2\epsilon_4^2}+1 \\
&= 1+\frac{1}{2(\alpha_{i^*}\beta_{i^*}-\alpha_i-\epsilon_3)^2} = M_2.
\end{aligned}
\end{equation}

After combining Case 1 and Case 2, \eqref{decompose_second_term} can be proved.

Then we prove \eqref{decompose_third_term}.
The first selection of arm $i$ contributes at most one to the left-hand side of \eqref{decompose_third_term}. For every later selection, $s = n_i(t-1) \geq 1$ and $\ol{\beta}_i(t) = S_{i,s}/s$ under the stack-of-rewards construction, and each $s \geq 1$ occurs at most once immediately before a selection of arm $i$. Hence, on every sample path,
\begin{align}
\label{eq:upper_bound_third_term_pathwise}
&\sum_{t=1}^T \id\left(i(t)=i, \overline{E_i^\mu(t)}\right) \nonumber \\
&\leq 1 + \sum_{t=2}^{T} \id\left(|\ol{\alpha}_i(t)-\alpha_i|>\epsilon_1\right) + \sum_{s=1}^{T-1} \id\left(\left|\tfrac{S_{i,s}}{s}-\beta_i\right|>\epsilon_1\right).
\end{align}
For each deterministic $t \geq 2$, $\ol{\alpha}_i(t)$ is the average of $t-1$ i.i.d.\ samples, and for each deterministic $s \geq 1$, $S_{i,s}/s$ is the average of $s$ i.i.d.\ samples. Taking expectations in \eqref{eq:upper_bound_third_term_pathwise} and applying Hoeffding's inequality to each term gives
\begin{align}
\label{eq:upper_bound_third_term}
&\sum_{t=1}^T \operatorname{Pr}\left(i(t)=i, \overline{E_i^\mu(t)}\right) \nonumber \\
&\leq 1 + 2\sum_{m=1}^{\infty} 2 e^{-2m\epsilon_1^2}
= 1 + \frac{4}{e^{2\epsilon_1^2}-1}
\leq 1+\frac{2}{\epsilon_1^2},
\end{align}
where the last step uses $e^{u}-1 \geq u$ for $u>0$.

Similar to \eqref{eq:alpha_is_small}, we can upper bound \eqref{decompose_fourth_term} as
\begin{align}
\label{eq:upper_bound_fourth_term}
\sum_{t=1}^T \Pr\left( \overline{G_{i^*}^\mu(t)}\right) \leq  1+\frac{1}{\epsilon_2^2},
\end{align}  

\end{document}